\documentclass{article} 
\usepackage{iclr2023_conference,times}

\usepackage{amsmath,amsfonts,bm}

\def\eqref#1{equation~\ref{#1}}

\def\1{\bm{1}}

\DeclareMathAlphabet{\mathsfit}{\encodingdefault}{\sfdefault}{m}{sl}
\SetMathAlphabet{\mathsfit}{bold}{\encodingdefault}{\sfdefault}{bx}{n}

\usepackage[utf8]{inputenc} 
\usepackage[T1]{fontenc}    
\usepackage{hyperref}       
\usepackage{url}            
\usepackage{booktabs}       
\usepackage{amsfonts}       
\usepackage{nicefrac}       
\usepackage{xcolor}         

\usepackage[pdftex]{graphicx}
\usepackage{float}
\usepackage{amsmath}
\usepackage{amssymb}
\usepackage{amsthm}
\usepackage{wrapfig}
\usepackage{subfig}
\usepackage{soul}
\usepackage[ruled]{algorithm2e}
\usepackage{multirow}
\usepackage{chngpage}
\usepackage{setspace}
\usepackage{enumitem}
\usepackage[figuresright]{rotating}

\newcommand{\fullname}{Hidden-Utility Self-Play}
\newcommand{\name}{HSP}
\newtheorem{theorem}{Theorem}[section]
\newtheorem{lemma}{Lemma}[section]

\title{Learning Zero-Shot Cooperation with Humans, Assuming Humans Are Biased}

\author{Chao Yu$^{1*}$,  
Jiaxuan Gao$^{1,2}$\thanks{Equal contribution} \hspace{0.1mm}, Weilin Liu$^{1}$, Botian Xu$^{2}$, Hao Tang$^{1}$, Jiaqi Yang$^{3}$, 
Yu Wang$^{1}$\footnotemark[2] \hspace{0.1mm}, 
Yi Wu$^{1,2,4}$\thanks{Corresponding Author} 
\\ [0.5ex]
$^1$ Tsinghua University, $^2$ Shanghai Artificial Intelligence Laboratory, $^3$ UC Berkeley, $^4$ Shanghai Qi Zhi Institute  \\ [0.5ex]
\texttt{zoeyuchao@gmail.com}
}

\iclrfinalcopy 
\begin{document}

\maketitle

\vspace{-4mm}
\begin{abstract}
There is a recent trend of applying multi-agent reinforcement learning (MARL) to train an agent that can cooperate with humans in a zero-shot fashion without using any human data. The typical workflow is to first repeatedly run self-play (SP) to build a policy pool and then train the final adaptive policy against this pool. A crucial limitation of this framework is that every policy in the pool is optimized w.r.t. the environment reward function, which implicitly assumes that the testing partners of the adaptive policy will be precisely optimizing the same reward function as well. However, human objectives are often substantially biased according to their own preferences, which can differ greatly from the environment reward. We propose a more general framework, \emph{\fullname} (\name), which explicitly models human biases as hidden reward functions in the self-play objective. By approximating the reward space as linear functions, {\name} adopts an effective technique to generate an augmented policy pool with biased policies. We evaluate {\name} on the \emph{Overcooked} benchmark. Empirical results show that our {\name} method produces higher rewards than baselines when cooperating with learned human models, manually scripted policies, and real humans. The {\name} policy is also rated as the most assistive policy based on human feedback.
\end{abstract}
\vspace{-4mm}
\section{Introduction}
\vspace{-2mm}

Building intelligent agents that can interact with, cooperate and assist humans remains a long-standing AI challenge with decades of research efforts~\citep{klien2004ten,ajoudani2018progress,dafoe2021cooperative}. Classical approaches are typically model-based, which (repeatedly) build an effective behavior model over human data and plan with the human model~\citep{sheridan2016human,carroll2019utility,bobu2020less}. 
Despite great successes, this model-based paradigm requires an expensive and time-consuming data collection process, which can be particularly problematic for complex problems tackled by today's AI techniques~\citep{kidd2008robots,biondi2019human} and may also suffer from privacy issues~\citep{pan2019you}. 

Recently, multi-agent reinforcement learning (MARL) has become a promising approach for many challenging decision-making problems. Particularly in competitive settings, AIs developed by MARL algorithms based on self-play (SP) defeated human professionals in a variety of domains~\citep{silver2018general,vinyals2019grandmaster,berner2019dota}.
This empirical evidence suggests a new direction of developing strong AIs that can directly cooperate with humans in a similar ``model-free'' fashion, i.e., via self-play. 
 
Different from zero-sum games, where simply adopting a Nash equilibrium strategy is sufficient, an obvious issue when training cooperative agents by self-play is \emph{convention overfitting}. Due to the existence of a large number of possible optimal strategies in a cooperative game, SP-trained agents can easily converge to a particular optimum and make decisions solely based on a specific behavior pattern, i.e., \emph{convention}~\citep{lowe2019pitfalls,hu2020other}, of its co-trainers, leading to poor generalization ability to unseen partners. 
To tackle this problem, recent works proposed a two-staged  framework by first developing a diverse policy pool consisting of multiple  SP-trained policies, which possibly cover different conventions, and then further training an adaptive policy against this policy pool~\citep{lupu2021trajectory,strouse2021collaborating,zhao2021maximum}. 

Despite the empirical success of this two-staged framework, a fundamental drawback exists. Even though the policy pool prevents convention overfitting, each SP-trained policy in the pool remains a solution, which is either optimal or sub-optimal, to a fixed reward function specified by the underlying cooperative game. This implies a crucial generalization assumption that any test-time partner will be \emph{precisely} optimizing the specified game reward. Such an assumption results in a pitfall in the case of cooperation with \emph{humans}. 
Human behavior has been widely studied in cognitive science~\citep{griffiths2015manifesto}, economics~\citep{wilkinson2017introduction} and game theory~\citep{fang2021introduction}. Systematic research has shown that humans' utility functions can be substantially biased even when a clear objective is given~\citep{pratt1978risk,selten1990bounded,camerer2011behavioral,barberis2013thirty}, suggesting that human behaviors may be subject to an unknown reward function that is very different from the game reward~\citep{nguyen2013analyzing}.
This fact reveals an algorithmic limitation of the existing SP-based methods. 

In this work, we propose \emph{\fullname} ({\name}), which extends the SP-based two-staged framework to the assumption of biased humans. {\name} explicitly models the human bias via an additional hidden reward function in the self-play training objective. 
Solutions to such a generalized formulation are capable of representing any non-adaptive human strategies. We further present a tractable approximation of the hidden reward function space and perform a random search over this approximated space when building the policy pool in the first stage. Hence, the enhanced pool can capture a wide range of possible human biases beyond  conventions~\citep{hu2020other,zhao2021maximum} and skill-levels~\citep{dafoe2021cooperative} w.r.t. the game reward. Accordingly, the final adaptive policy derived in the second phase can have a much stronger adaptation capability to unseen humans. 

We evaluate {\name} in a popular human-AI cooperation benchmark, \emph{Overcooked}~\citep{carroll2019utility}, which is a fully observable two-player cooperative game.
We conduct comprehensive ablation studies and comparisons with baselines that do not explicitly model human biases. 
Empirical results show that {\name} achieves superior performances when cooperating with behavior models learned from human data. 
In addition, we also consider a collection of manually scripted biased strategies, which are ensured to be sufficiently distinct from the policy pool, and {\name} produces an even larger performance improvement over the baselines.
Finally, we conduct real human studies. Collected feedbacks show that the human participants consistently feel that the agent trained by {\name} is much more assistive than the baselines.
We emphasize that, in addition to algorithmic contributions, our empirical analysis, which considers learned models, script policies and real humans as diverse testing partners,  also provides a more thorough evaluation standard for learning human-assistive AIs.
\vspace{-4mm}
\section{Related Work}
\vspace{-2mm}
There is a broad literature on improving the zero-shot generalization ability of MARL agents to unseen partners~\citep{kirk2021survey}. 
Particularly for cooperative games, this problem is often called \emph{ad hoc team play}~\citep{stone2010ad} or \emph{zero-shot cooperation} (ZSC)~\citep{hu2020other}.
Since most existing methods are based on self-play~\citep{rashid2018qmix,yu2021surprising},
how to avoid convention overfitting becomes a critical challenge in ZSC. 
Representative works include improved policy representation~\citep{zhang2020multi,chen2020aateam}, randomization over invariant game structures~\citep{hu2020other,treutlein2021new}, population-based training~\citep{Long*2020Evolutionary,Lowe*2020On,cui2021k} and belief modeling for partial observable settings~\citep{hu2021off,xie2021learning}.
\emph{Fictitious co-play} (FCP)~\citep{strouse2021collaborating} proposes a two-stage framework by first creating a pool of self-play policies and their previous versions and then training an adaptive policy against them.
Some techniques improves the diversity of the policy pool~\citep{garnelo2021pick,liu2021towards,zhao2021maximum,lupu2021trajectory} for a stronger adaptive policy~\citep{knott2021evaluating}.
 
We follow the FCP framework and augment the policy pool with biased strategies.
Notably, techniques for learning a robust policy in competitive games, such as policy ensemble~\citep{lowe2017multi}, adversarial training~\citep{li2019robust} and double oracle~\citep{lanctot2017unified}, are complementary to our focus. 

Building AIs that can cooperate with humans remains a fundamental challenge in AI~\citep{dafoe2021cooperative}. 
A critical issue is that humans can be systematically biased~\citep{camerer2011behavioral,russell2019human}. Hence, great efforts have been made to model human biases, 
such as irrationality~\citep{selten1990bounded,bobu2020less,laidlaw2022the}, risk aversion~\citep{pratt1978risk,barberis2013thirty}, and myopia~\citep{evans2016learning}. Many popular models further assume humans have hidden subject utility functions~\citep{nguyen2013analyzing,hadfield2016cooperative,eckersley2019impossibility,shah2019feasibility}. Conventional methods for human-AI collaboration require an accurate behavior model over human data~\citep{ajoudani2018progress,kwon2020humans,kress2021formalizing,wang2022co}, while we consider the setting of no human data. 
Hence, we explicitly model human biases as a hidden utility function in the self-play objective to reflect possible human biases beyond conventions w.r.t. optimal rewards. We prove that such a hidden-utility model can represent any strategy of non-adaptive humans.  
Notably, it is also feasible to generalize our model to capture higher cognitive hierarchies~\citep{camerer2004cognitive}, which we leave as a future direction. 

We approximate the reward space by a linear function space over event-based features. Such a linear representation is typical in inverse reinforcement learning~\citep{ng2000algorithms}, policy transfer~\citep{barreto2017successor}, evolution computing~\citep{cully2015robots} and game theory~\citep{winterfeldt1975multi,kiekintveld2013security}. Event-based rewards are also widely adopted as a general design principle in robot learning~\citep{fu2018variational,zhu2019ingredients,ahn2022can}.
We perform randomization over feature weights to produce diverse biased strategies. Similar ideas 
have been adopted in other settings, such as generating adversaries~\citep{paruchuri2006security}, emergent team-formation~\citep{baker2020emergent}, and searching for diverse Nash equilibria in general-sum games~\citep{tang2020discovering}.
In our implementation, we use multi-reward signals as an approximate metric to filter out duplicated policies, which is inspired by the quality diversity method~\citep{pugh2016quality}.
There are also some works utilizing model-based methods to solve zero-shot cooperation~\cite{wu2021too}. Their focus is orthogonal to our approach since they focus more on constructing an adaptive agent, while our approach aims to find more diverse strategies. Besides, we adopt an end-to-end fashion to train an adaptive agent, which is more general.
Lastly, our final adaptive agent assumes a zero-shot setting without any data from its testing partner. This can be further extended by allowing meta-adaptation at test time~\citep{charakorn2021learning,gupta2021dynamic,nekoei2021continuous}, which we leave as a future direction. 

\vspace{-2mm}
\section{Preliminary}
\vspace{-2mm}
{\bf Two-Player Human-AI Cooperative Game:} A human-AI cooperative game is defined on a world model, i.e., a two-player Markov decision process denoted by $M=\langle \mathcal S, \mathcal A, P, R\rangle$, with one player with policy $\pi_A$ being an AI and the other with policy $\pi_H$ being a human. $\mathcal S$ is a set of world states. $\mathcal A$ is a set of possible actions for each player. $P$ is a transition function over states given the actions from both players. $R$ is a global reward function.
A policy $\pi_{i}$ produces an action $a^{(i)}_t\in\mathcal{A}$ given a world state $s_t\in\mathcal{S}$ at the time step $t$.
We use the expected discounted return $J(\pi_A,\pi_H)=\mathbb E_{s_t,a_t^{(i)}}\left[\sum_t\gamma^t R(s_t,a_t^{(A)},a_t^{(H)})\right]$ as the objective. 
Note that $J(\pi_H,\pi_A)$ can be similarly defined, and we use $J(\pi_A,\pi_H)$ for conciseness without loss of generality. 
Let $P_H:\Pi\rightarrow [0,1]$ be the unknown distribution of human policies. 
The goal is to find a policy $\pi_A$ that maximizes the expected return with an unknown human, i.e., $\mathbb E_{\pi_H\sim P_H}[J(\pi_H, \pi_A)]$\label{eq:true-obj}. 
In practice, many works construct or learn a policy distribution $\hat P_H$ to approximate real-world human behaviors, leading to an approximated objective for $\pi_A$, i.e., $\mathbb E_{\hat{\pi}_H\sim \hat P_H}[J(\pi_A,\hat{\pi}_H)]$. 

\textbf{Self-Play for Human-AI Cooperation:} Self-play (SP) optimizes $J(\pi_1,\pi_2)$ with two parametric policies $\pi_1$ and $\pi_2$ and takes $\pi_1$ as $\pi_A$ without use of human data.
However, SP suffers from poor generalization since SP converges to a specific optimum and overfits the resulting behavior convention. 
Population-based training (PBT) improves SP by representing $\pi_i$ as a mixture of $K$ individual policies $\{\pi_i^{(k)}\}_{k=1}^K$ and runs \emph{cross-play} between policies by optimizing the expected return~\citep{Long*2020Evolutionary,Lowe*2020On,cui2021k}.
PBT can be further improved by adding a diversity bonus over the population~\citep{garnelo2021pick,liu2021towards,lupu2021trajectory}. 

\textbf{Fictitious Co-Play (FCP):} FCP~\citep{strouse2021collaborating} is a recent work on zero-shot human-AI cooperation with strong empirical performances. FCP extends PBT via a two-stage framework. 
In the first stage, FCP trains $K$ individual policy pairs $\{(\pi_1^{(k)},\pi_2^{(k)})\}_{k=1}^K$ by optimizing $J(\pi_1^{(k)},\pi_2^{(k)})$ for each $k$. 
Each policy pair $(\pi_1^{(k)},\pi_2^{(k)})$ may quickly converge to a distinct local optimum. 
Then FCP constructs a policy pool $\Pi_2=\{\tilde{\pi}_2^{(k)},\pi^{(k)}_2\}_{k=1}^K$ with two past versions of each converged SP policy $\pi^{(k)}_2$, denoted by $\tilde{\pi}_2^{(k)}$.
In the second stage, FCP constructs a human proxy  distribution $\hat P_H$ by randomly sampling from $\Pi_2$ and trains $\pi_A$ by optimizing $\mathbb E_{\hat{\pi}_H\sim \hat P_H}[J(\pi_A,\hat{\pi}_H)]$. 
We remark that, for a better cooperation, the adaptive policy $\pi_A$ should condition on the state-action history in an episode to infer the intention of its partner.
Individual SP policies ensure $\hat P_H$ contains diverse conventions while using past versions enables $\hat P_H$ to cover different skill levels. So, the final policy $\pi_A$ can be forced to adapt to humans with unknown conventions or sub-optimalities.
Maximum Entropy Population-based Training (MEP)~\citep{zhao2021maximum} is the latest variant of FCP, which adopts the population entropy as a diversity bonus in the first stage to improve the generalization of the learned $\pi_A$.

\vspace{-2mm}
\section{Cooperating with Humans in \emph{Overcooked}: A Motivating Example}
\vspace{-2mm}
\label{sec:case-study}

\textbf{Overcooked Game:} \emph{Overcooked}~\citep{carroll2019utility} is a fully observable two-player cooperative game developed as a testbed for human-AI cooperation. In Overcooked, players cooperatively accomplish different soup orders and serve the soups for rewards. Basic game items include onions, tomatoes, and dishes. 
An agent can move in the game or ``interact'' to trigger some events, such as grabbing/putting an item, serving soup, etc., depending on the game state.  
To finish an order, players should put a proper amount of ingredients into the pot and cook for some time. Once a soup is finished, players should pick up the soup with a dish and serve it to get a reward. 
Different orders have different cooking times and different rewards. 
Fig.~\ref{fig:layouts} demonstrates five layouts we consider, where the first three onion-only layouts are adopted from \citep{carroll2019utility}, 
while the latter two, \emph{Distant Tomato} and \emph{Many Orders}, are newly introduced to include tomato orders to make the problem more challenging: an AI needs to carefully adapt its behavior to either cook onions or tomatoes according to the other player's actions. 

\begin{figure}[bt]
\vspace{-13mm}
	\centering
	{\centering
        {\includegraphics[width=1\textwidth]{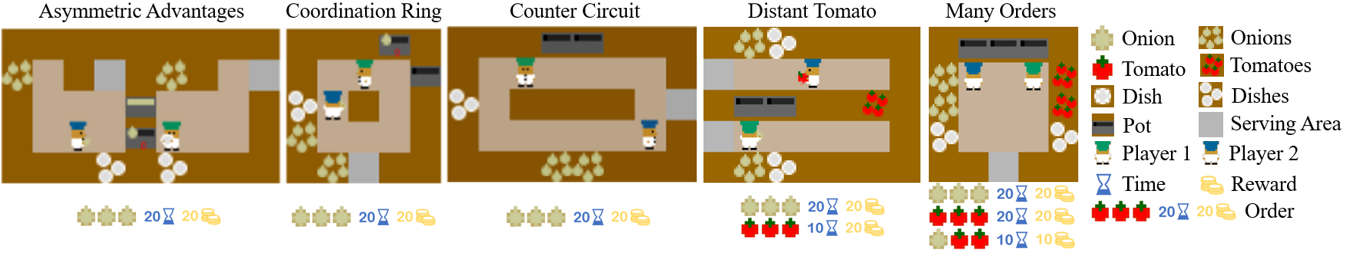}
    	}
    }
	\centering 
	\vspace{-6mm}
	\caption{\textbf{Layouts in Overcooked.} From left to right are \emph{Asymmetric Advantages}, \emph{Coordination Ring}, \emph{Counter Circuit}, \emph{Distant Tomato} and \emph{Many Orders} respectively, with orders shown below. }
\label{fig:layouts}
\vspace{-5mm}
\end{figure}

\textbf{A Concrete Example of Human Preference:} Fig.~\ref{fig:fcp-far_tomato} illustrates a motivating example in \emph{Distant Tomato} (the 4th layout in Fig.~\ref{fig:layouts}). There are two orders: one requires three onions, and the other requires three tomatoes. We run FCP on this multi-order scenario, and all the policies in the FCP policy pool converge to the specific pattern of only cooking onion soup (Fig.~\ref{fig:fcp-far_tomato-a}). Hence, the final adaptive policy by FCP only learns to grab onions and cook onion soups.
Cooking tomato soup is a sub-optimal strategy that requires many extra moves, so the onion-only policy pool is exactly the solution to the FCP self-play objective under the environment reward. 
However, it is particularly reasonable for a human to dislike onions and accordingly only grab tomatoes in a game.
To be an assistive AI, the policy should adapt its strategy to follow the human preference for tomatoes. On the contrary, as shown in Fig.~\ref{fig:fcp-far_tomato-b}, the FCP policy completely ignores human moves for tomatoes and even results in poor cooperation by producing valueless wrong orders of mixed onions and tomatoes. 
Thus, to make an FCP agent human-assistive, the first-stage policy pool should not only contain optimal strategies (i.e., onion soups) of different conventions but also cover diverse human preferences (e.g., tomatoes) even if these preferences are sub-optimal under the environment reward.

\begin{figure}[h]
\vspace{-6mm}
	\centering
	
    \subfloat[\label{fig:fcp-far_tomato-a}FCP-FCP]{
		\includegraphics[height=1.8cm]{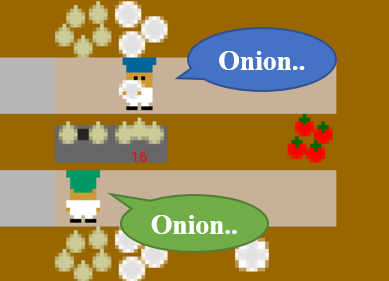}
		}
    \subfloat[\label{fig:fcp-far_tomato-b}FCP's failure case when cooperating with a human player]{
		\includegraphics[height=1.8cm]{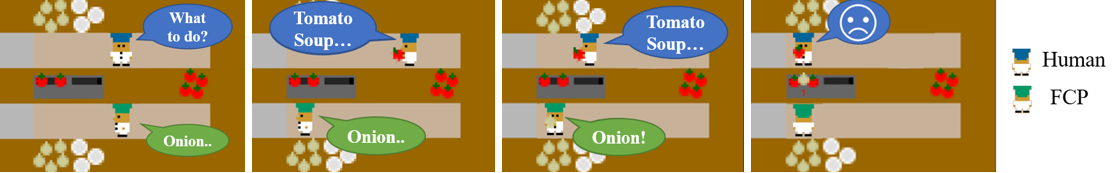}
		}
	\centering 
	\vspace{-2mm}
	\caption{Motivating example. (a) FCP converges to the optimal onion soup strategy. (b) A failure case of FCP with a human partner: FCP agent corrupts the human's plan of cooking tomato soups.}
    \label{fig:fcp-far_tomato}
\vspace{-3mm}
\end{figure}
\vspace{-2mm}
\section{Methodology}
\label{sec:method}
\vspace{-2mm}

We introduce a general formulation to model human preferences and develop a tractable learning objective (Sec.~\ref{sec:reward-space}). The algorithm, \emph{Hidden-Utility Self-Play} (\name), is summarized in Sec.~\ref{sec:hsp}.

\subsection{Hidden-Utility Markov Game}
\label{sec:method-formulation}

The key insight from Sec.~\ref{sec:case-study} is that humans may not truthfully behave under the environment reward.  
Instead, humans are biased and driven by their own utility functions, which are formulated below. 

\textbf{Definition:} A two-player \emph{hidden utility Markov game} is defined as $\langle\mathcal S,\mathcal A, P, R_w, R_t\rangle$. $\langle\mathcal S,\mathcal A, P, R_t\rangle$ corresponds to the original game MDP with $R_t$ being the task reward function. $R_w$ denotes an \emph{additional} hidden reward function. 
There are two players, $\pi_a$, whose goal is to maximize the task reward $R_t$, and $\pi_w$, whose goal is to maximize the hidden reward $R_w$. $R_w$  is only visible to $\pi_w$.

Let $J(\pi_1,\pi_2|R)$ denote the expected return under reward $R$ with a policy  $\pi_1$ and $\pi_2$.  During self-play, $\pi_a$ optimizes $J(\pi_a,\pi_w|R_t)$ while $\pi_w$ optimizes $J(\pi_a, \pi_w|R_w)$.
A solution policy profile $(\pi_a^*,\pi_w^*)$ to the hidden utility Markov game is now defined by a Nash equilibrium (NE): $J(\pi_a^*, \pi_w^*|R_w)\ge J(\pi_a^*, \pi_w'|R_w),\forall \pi_w'$ and $J(\pi_a^*,\pi_w^*|R_t)\ge J(\pi_a',\pi_w^*|R_t),\forall \pi_a'$.$\hfill \square$

Intuitively, with a suitable hidden reward function $R_w$, we can obtain any possible (non-adaptive and consistent) human policy by solving the hidden-utility game induced by $R_w$.
\begin{lemma}
    \label{theorem:hidden-reward}
    Given an MDP $M=\langle \mathcal S,\mathcal A,P,R_t\rangle$, for any policy $\pi:\mathcal S\times \mathcal A\rightarrow [0,1]$, there exists a hidden reward function $R_w$ such that the two-player hidden utility Markov game $M'=\langle\mathcal S,\mathcal A,P,R_w,R_t\rangle$ has a Nash equilibrium $(\pi_a^*,\pi_w^*)$ where $\pi_w^*=\pi$.
\end{lemma}

Lemma \ref{theorem:hidden-reward} connects any human behavior to a hidden reward function. 
Then the objective of the adaptive policy $\pi_A$ in  Eq.~(\ref{eq:true-obj}) can be formulated under the hidden reward function space  $\mathcal R$ as follows.

\begin{theorem}
    \label{theorem:hidden-reward-space}

    For any $\epsilon >0$, there exists a mapping $\tilde{\pi}_w$ where $\tilde{\pi}_w(R_w)$ denotes the derived policy $\pi_w^*$ in the NE of the hidden utility Markov game $M_w=\langle\mathcal S,\mathcal A, P, R_w, R_t\rangle$ induced by $R_w$,
    and a distribution $P_R:\mathcal R\rightarrow [0, 1]$ over the hidden reward space $\mathcal R$, such that, for any adaptive policy $\pi_A\in \arg\max_{\pi'} \mathbb E_{R_w\sim P_R}[J(\pi', \tilde \pi_w(R_w))]$, $\pi_A$ approximately maximizes the ground-truth objective with at most an $\epsilon$ gap, i.e., $\mathbb E_{\pi_H\sim P_H}[J(\pi_A,\pi_H)]\ge \max_{\pi'}\mathbb E_{\pi_H\sim P_H}[J(\pi',\pi_H)]-\epsilon$.
\end{theorem}

Theorem \ref{theorem:hidden-reward-space} indicates that it is possible to derive diverse human behaviors by properly designing a hidden reward distribution $\hat P_R$, which can have a much lower intrinsic dimension than the policy distribution.
In \emph{Overcooked}, human preferences can be typically described by a few features, such as interaction with objects or certain type of game events, like finishing an order or delivering a soup.
By properly approximating the hidden reward distribution as $\hat P_R$, the learning objective becomes,
\begin{align}\label{eq:hidden-reward-obj}
    \pi_A=\arg\max_{\pi'}\mathbb E_{R_w\sim \hat P_R}[J(\pi', \tilde \pi_w(R_w))] 
\end{align}
Eq.~(\ref{eq:hidden-reward-obj}) naturally suggests a two-staged solution by first constructing a policy pool $\{\tilde{\pi}_w(R):R\sim \hat{P}_R\}$ from $\hat{P}_R$ and then training $\pi_A$ to maximize the game reward w.r.t. the induced pool.

\subsection{Construct a Policy Pool of Diverse Preferences}
\label{sec:reward-space}

{\bf Event-based Reward Function Space: } 
The fundamental question is how to design a proper hidden reward function space $\mathcal R$. In general, a valid reward space is intractably large. Inspired by the fact that human preferences are often event-centric, we formulate $\mathcal{R}$ as linear functions over event features, 
namely $\mathcal R=\{R_w:R_w(s, a_1, a_2)=\phi(s, a_1, a_2)^Tw, ||w||_{\infty}\le C_{\max}\}$.  $C_{\max}$ is a bound on the feature weight $w$ while $\phi:\mathcal S\times \mathcal A\times\mathcal A\rightarrow \mathbb R^{m}$ specifies occurrences of different game events when taking joint action $(a_1,a_2)$ at state $s$. 

{\bf Derive a Diverse Set of Biased Policies:} We simply perform a random search over the feature weight $w$ to derive a set of diverse behaviors. 
We first draw $N$ samples $\{w^{(i)}\}_{i\in[N]}$ for the feature weight $w$ where $w^{(i)}_j$ is sampled uniformly from a set of values $C_j$, 
leading to a set of hidden reward functions $\{R_w^{(i)}:R_w^{(i)}(s,a_1,a_2)=\phi(s,a_1,a_2)^Tw^{(i)}\}_{i\in[N]}$. For each hidden reward function $R_w^{(i)}$, we find an approximated NE, $\pi_w^{(i)}, \pi_a^{(i)}$, of the hidden utility Markov game induced by $R_w^{(i)}$ through self-play. 
The above process produces a policy pool $\{\pi_w^{(i)}\}_{i\in [N]}$ that can cover a wide range of behavior preferences.

\begin{wrapfigure}{R}{0.42\textwidth}
\vspace{-5mm}
\small{
\begin{algorithm}[H]
\caption{Greedy Policy Selection}\label{algo:greedy}
$S\leftarrow \{i_0\}$ where $i_0\sim[N]$\;
\For{$i=1\rightarrow K-1$}{
    $k'\leftarrow \arg\max_{k'\notin S}\text{ED}(S\cup\{k'\})$\;
    $S\leftarrow S\cup\{k'\}$\;
}
\end{algorithm}}
\vspace{-5mm}
\end{wrapfigure}

\textbf{Policy Filtering:} We notice that the derived pool often contains a lot of similar policies. 
This is because the same policy can be optimal under a set of reward functions, which is typical in multi-objective optimization~\citep{chugh2019survey, tabatabaei2015survey}. 
Duplicated policies simply slow down training without any help to learn $\pi_A$. 
For more efficient training, we adopt a behavior metric, i.e., \emph{event-based diversity}, to only keep distinct ones from the initial pool. 
For each biased policy $\pi_w^{(i)}$, let $\text{EC}^{(i)}$ denote the expected event count, i.e. $\mathbb E[\sum_{t=1}^T \phi(s_t,a_t)|\pi_w^{(i)},\pi_a^{(i)}]$. 
We define event-based diversity for a subset $S\subseteq [N]$ by normalized pairwise EC differences, i.e., $\text{ED}(S)=\sum_{i,j\in S}\sum_k c_k\cdot |\text{EC}_k^{(i)} - \text{EC}_k^{(j)}|$, where $c_k$ is a frequency normalization constant. 
Finding a subset $S^*$ of size $K$ with the optimal $\text{ED}$ can be expensive. We simply adopt a greedy method in Algo.~\ref{algo:greedy} to select policies incrementally.

\begin{wrapfigure}{r}{0.435\textwidth}
\vspace{-15mm}
\small{
\begin{algorithm}[H]
\caption{Hidden-Utility Self-Play}\label{algo:random}
\For{$i=1\rightarrow N$}{
    Train $\pi_w^{(i)}$ and $\pi_a^{(i)}$ under sampled $R_w^{(i)}$\;
}
Run Algo.~\ref{algo:greedy} to only keep $K$ policies\;
Initial policy $\pi_A$\;
\Repeat{enough iterations}{
    Rollout with $\pi_A$ and sampled $\pi_w^{(i)}$\;
    Update $\pi_A$\;
}
\end{algorithm}}
\vspace{-10mm}
\end{wrapfigure}

\subsection{Hidden-Utility Self-Play}
\label{sec:hsp}
Given the filtered policy pool, we train the final adaptive policy $\pi_A$ over rollout games by $\pi_A$ and randomly sampled policies from the pool, which completes our overall algorithm {\name} in Algo.~\ref{algo:random}.

We implement {\name} using MAPPO~\citep{yu2021surprising} as the RL algorithm.
In the first stage, we use MLP policies for fast SP convergence.
In practice, we use half of the policy pool to train biased policies and the other half to train MEP policies~\citep{zhao2021maximum} under the game reward.
This increases the overall pool towards the game reward, leading to improved empirical performances. 
For the final adaptive training, as suggested in \citep{tang2020discovering}, we add the identity of each biased policy as an additional feature to the critic. 
For event-based features for the reward space, we consider event types, including interactions with basic items and events causing non-zero rewards in \emph{Overcooked}. 
Full implementation details can be found in Appendix \ref{sec:app-implement} and \ref{sec:app-training}.
\vspace{-2mm}

\section{Experiments}
\vspace{-2mm}
{\bf Baselines.} We compare {\name} with other SP-based baselines, including Fictitious Co-Play (FCP), Maximum Entropy Population-based training (MEP), and Trajectory Diversity-based PBT (TrajDiv). All methods follow a two-stage framework with a final pool size of $36$, which we empirically verified to be sufficiently large to avoid performance degradation for all methods. More analysis on pool size can be found in Appendix \ref{sec:app-poolsize}. The implementation details of baselines can be found in Appendix \ref{sec:app-baselines}.
Each policy is trained for 100M timesteps for convergence over 5 random seeds.
Full training details with hyper-parameter settings can be found in Appendix \ref{sec:app-hyperparameters}.

{\bf Evaluation.}  
We aim to examine whether {\name} can cooperate well with (1) \emph{learned human models}, (2) \emph{scripted} policies with strong preferences, and (3) \emph{real humans}.
We use both game reward and human feedback as evaluation metrics. We remark that since a biased human player may play a sub-optimal strategy, the game reward may not fully reflect the performance gap between the baselines and {\name}. Our goal is to ensure the learned policy is effective for biased partners/humans. Therefore, we consider human feedback as the fundamental metric. Ablation studies are also performed to investigate the impact of our design choices in {\name}. In tables, maximum returns or comparable returns within a threshold of 5 are marked in bold. Full results can be found in Appendix \ref{sec:app-results}.
\vspace{-2mm}
\subsection{Cooperation with Learned Human Models in Onion-Only Layouts}
\vspace{-2mm}

For evaluation with learned human models, we adopted the models provided by \citep{carroll2019utility}, which \emph{only support onion-only layouts}, including Asymm. Adv., Coord. Ring and Counter Circ.. The results are shown in Tab.~\ref{tab:human-proxy}. For a fair comparison, we reimplement all the baselines, labeled \emph{MEP, FCP, and TrajDiv}, with the same training steps and policy pool size as {\name}. We additionally take the best performance ever reported in the existing literature, labeled \emph{Existing SOTA} in Tab.~\ref{tab:human-proxy}. Our implementation achieves substantially higher scores than Existing SOTA when evaluated with the same human proxy models. {\name} further outperforms other reimplementations in Asymm. Adv. and is comparable with the best baseline in the rest. Full results of the evaluation with learned human models can be found in Appendix \ref{sec:app-results-learned}. We emphasize that the improvement is marginal because the learned human models have limited representation power to imitate natural human behaviors, which typically cover many behavior modalities. Fig.\ref{fig:human-imitate-traj} in Appendix \ref{sec:app-evidence} shows trajectories induced by the learned human models only cover a narrow subspace of trajectories played by human players. Further analysis of the learned human models can be found in Appendix \ref{sec:app-evidence}. Furthermore, our implementation of baselines achieves substantially better results than the original papers~\citep{carroll2019utility,zhao2021maximum}, which also makes the improvement margin smaller.

\begin{table}[ht]
\vspace{-4mm}
\centering
\begin{tabular}{ccccc} 
\toprule
 & Pos. & Asymm. Adv.           & Coord. Ring          & Counter Circ.          \\ 
\midrule
\multirow{2}{*}{Existing SOTA}     & 1    & 141.1\scriptsize{(12.5)}          & 92.7\scriptsize{(7.4)}          & 54.5\scriptsize{(2.3)}            \\
                                     & 2    & 84.6\scriptsize{(16.3)}          & 107.3\scriptsize{(6.4)}          & 55.8\scriptsize{(3.6)}            \\ 
\hline
\multirow{2}{*}{FCP}     & 1    & 282.8\scriptsize{(9.4)}          & \textbf{161.3}\scriptsize{(1.6)}          & 95.9\scriptsize{(2.0)}            \\
                                     & 2    & 203.8\scriptsize{(8.2)}          & \textbf{161.0}\scriptsize{(2.7)}          & 92.7\scriptsize{(1.3)}            \\ 
\hline
\multirow{2}{*}{MEP}     & 1    & 291.7\scriptsize{(4.6)}          & \textbf{161.8}\scriptsize{(0.7)} & \textbf{108.8}\scriptsize{(4.2)}              \\
                                     & 2    & 203.4\scriptsize{(2.0)}          & \textbf{164.2}\scriptsize{(2.1)} & \textbf{111.1}\scriptsize{(0.7)}              \\ 
\hline
\multirow{2}{*}{TrajDiv} & 1    & 289.3\scriptsize{(8.8)}          & 150.8\scriptsize{(3.1)}          & 60.1\scriptsize{(5.0)}              \\
                                     & 2    & 194.2\scriptsize{(0.7)}          & 142.1\scriptsize{(2.3)}          & 53.7\scriptsize{(12.4)}              \\ 
\hline
\multirow{2}{*}{HSP}     & 1    & \textbf{300.3}\scriptsize{(2.2)} & \textbf{160.0}\scriptsize{(2.6)}          & \textbf{107.4}\scriptsize{(3.5)}  \\
                                     & 2    & \textbf{217.1}\scriptsize{(3.3)} & \textbf{160.6}\scriptsize{(3.3)}          & \textbf{106.6}\scriptsize{(3.0)}           \\
\bottomrule
\end{tabular}
\vspace{-2mm}
\caption{Comparison of average episode reward and standard deviation when cooperating with learned human models. The Pos. column indicates the roles played by AI policies. \textit{Existing SOTA} is the best performance ever reported in the existing literature. {\name} achieves substantially higher scores than Existing SOTA. And {\name} further outperforms other methods in Asymm. Adv. and is comparable with the best baseline in the rest.}
\label{tab:human-proxy}
\vspace{-6mm}
\end{table}

\vspace{-2mm}
\subsection{Ablation Studies}
\vspace{-2mm}
\begin{wrapfigure}{r}{0.4\textwidth}
	\centering
	\vspace{-3mm}
	\hspace{-2mm}
		\includegraphics[width=0.4\textwidth]{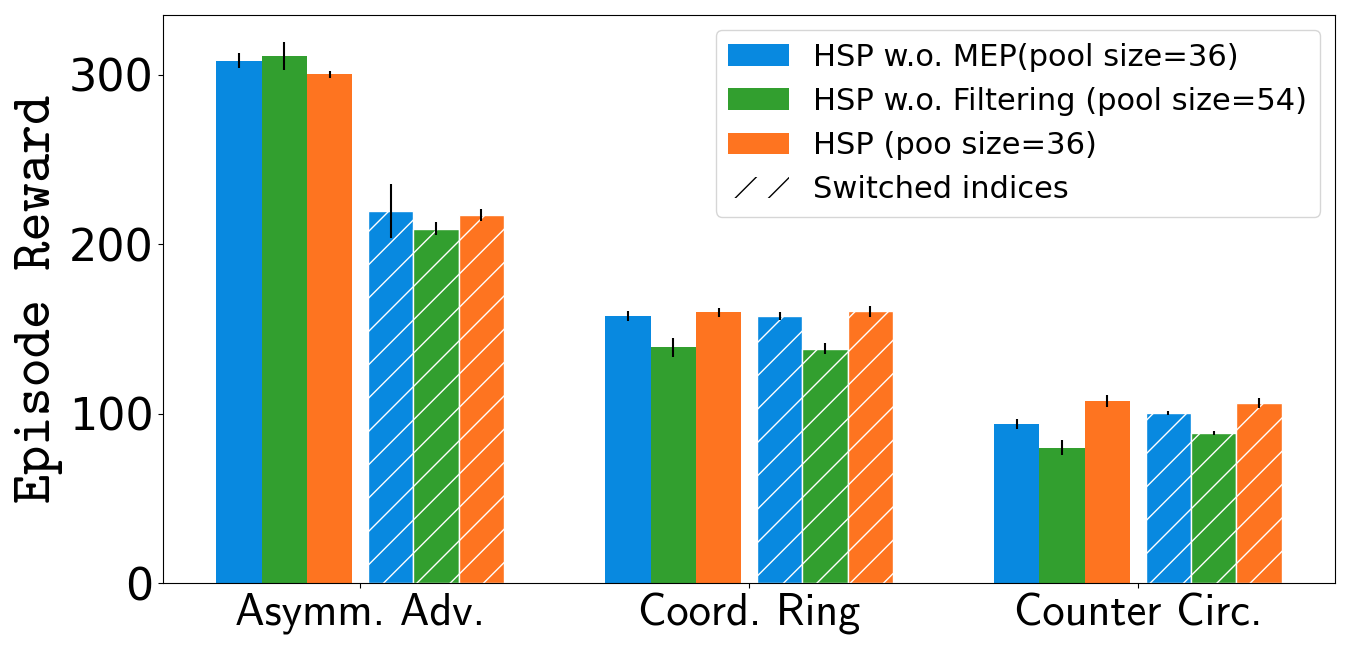}
	\centering 
	\vspace{-2mm}
        \caption{Performance of different pool construction strategies. Results suggest that it is beneficial to incorporate MEP policies and filter duplicated policies.}
	\vspace{-6mm}
    \label{fig:human-proxy-ablation}
\end{wrapfigure}
We investigate the impact of our design choices, including the construction of the final policy pool and the batch size for training the adaptive policy. 

\textbf{Policy Pool Construction:} {\name} has two techniques for the policy pool, i.e., (1) policy filtering to remove duplicated biased policies and (2) the use of MEP policies under the game reward for half of the pool size. 
We measure the performance with human proxies by turning these options off. For ``\emph{HSP w.o. Filtering}'', we keep all policies by random search in the policy pool, resulting in a larger pool size of 54 (18 MEP policies and a total of 36 random search ones). For``\emph{HSP w.o. MEP}'', we exclude MEP policies from the policy pool and keep all biased policies without filtering, which leads to the same pool size of 36. The results are shown in Fig.~\ref{fig:human-proxy-ablation} and the detailed numbers can be found in Appendix \ref{sec:app-poolconstruction}.
By excluding MEP policies, the HSP variant (\emph{HSP w.o. MEP}) performs worse in the more complicated layout Counter Circ. while remaining comparable in the other two simpler ones. 
So we suggest including a few MEP policies when possible.
With policy filtering turned off, even though the policy pool size grows, the performance significantly decays in both Coord. Ring and Counter Circ. layouts, suggests that duplicated biased policies can hurt policy generalization.

\textbf{Batch Size:} We measure the training curves of the final adaptive policy under the game reward using different numbers of parallel rollout threads in MAPPO. More parallel threads indicate a larger batch size. 
The results in all five layouts are reported in Fig.~\ref{fig:parallel}.
In general, we observe that a larger batch size often leads to better training performance. In particular, when the batch size is small, i.e., using 50 or 100 parallel threads, training becomes significantly unstable and even breaks in three layouts. 
Note that the biased policies in the {\name} policy pool have particularly diverse behaviors, which cause a high policy gradient variance when training the final adaptive policy. Therefore, a sufficiently large training batch size can be critical to stable optimization. We adopt 300 parallel threads in all our experiments for a fair comparison. 
\begin{figure}[H]
\vspace{-4mm}
	\centering
	{	\includegraphics[width=\linewidth]{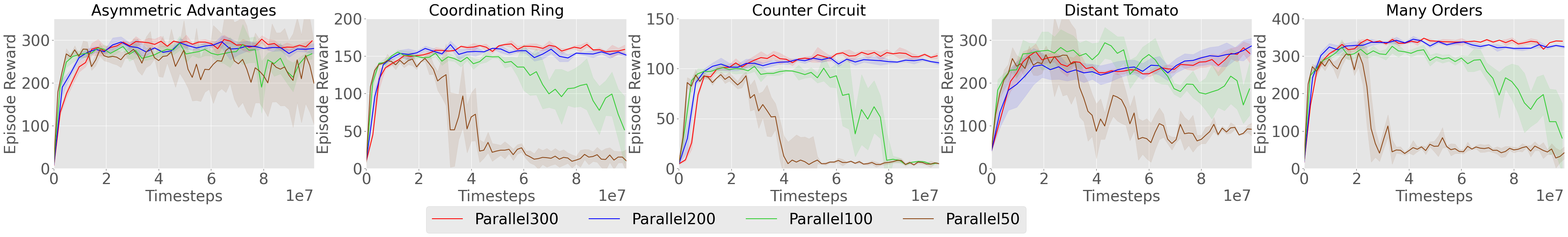}
		}
	\centering 
	\vspace{-5mm}
	\caption{Average game reward by using different numbers of parallel rollout threads in MAPPO to train the final adaptive policy. More parallel threads imply a larger training batch size.}
    \label{fig:parallel}
\vspace{-5mm}
\end{figure}

\textbf{Practical Remark:} Overall, we suggest using a pool size of 36 and including a few MEP policies for the best empirical performance. Besides, a sufficiently large training batch size can help stable optimization, and we use the same batch size for all methods for a fair comparison.

\subsection{Cooperation with Scripted Policies with Strong Behavior Preferences}

We empirically notice that human models learned by imitating the entire human trajectories cannot well capture a wide range of behavior modalities. 
So, we manually designed a set of script policies to encode some particular human preferences: 
\emph{Onion/Tomato Placement}, which continuously places onions or tomatoes into the pot,
\emph{Onion/Dish Everywhere}, which keeps putting onions or dishes on the counters,
\emph{Tomato/Onion Placement and Delivery}, which puts tomatoes/onions into the pot in half of the time and tries to deliver soup in the other half of the time. For a fair comparison, we ensure that \emph{all scripted policies are strictly different from the {\name} policy pool}. More details about scripted policies and a full evaluation can be found in Appendix \ref{sec:app-script}.

We remark that scripted policies are only used for evaluation but not for training {\name}. Tab.~\ref{tab:Average-reward-script} shows the average game reward of all the methods when paired with scripted policies, where {\name} significantly outperforms all baselines.
In particular, in \emph{Distant Tomato}, when cooperating with a strong tomato preference policy ({Tomato Placement}), {\name} achieves a $10\times$ higher score than other baselines, 
suggesting that the tomato-preferred behavior is well captured by {\name}.
\vspace{-2mm}
\begin{table}[ht]
\vspace{-13mm}
\centering
\begin{tabular}{cccccc} 
\toprule
                             & Scripts     & FCP                 & MEP                  & TrajDiv     & HSP                   \\ 
\midrule
\multirow{2}{*}{Asymm. Adv.}   & Onion Placement   & 334.8\scriptsize{(13.0)}         & 330.5\scriptsize{(14.2)}          & 323.6\scriptsize{(17.0)} & \textbf{376.8}\scriptsize{(9.9)}   \\
                             & Onion Place.\&Delivery  & \textbf{297.7}\scriptsize{(3.4)}          & \textbf{298.5}\scriptsize{(3.4)}           & \textbf{290.0}\scriptsize{(4.7)}  & \textbf{300.1}\scriptsize{(4.1)}   \\ 
\hline
\multirow{2}{*}{Coord. Ring}  & Onion Everywhere & 109.1\scriptsize{(7.9)}          & \textbf{124.0}\scriptsize{(3.4)}  & 116.9\scriptsize{(8.9)}  & \textbf{121.2}\scriptsize{(12.6)}           \\
                             & Dish Everywhere & 94.4\scriptsize{(3.8)}           & 100.2\scriptsize{(5.3)}           & 107.3\scriptsize{(5.3)}  & \textbf{115.4}\scriptsize{(7.4)}   \\ 
\hline
\multirow{2}{*}{Counter Circ.} & Onion Everywhere & 63.7\scriptsize{(9.2)}           & 88.9\scriptsize{(5.1)}            & 82.0\scriptsize{(12.8)}  & \textbf{107.5}\scriptsize{(3.5)}   \\
                             & Dish Everywhere & 57.0\scriptsize{(5.3)}           & 53.0\scriptsize{(1.8)}            & 57.2\scriptsize{(2.2)}   & \textbf{78.5}\scriptsize{(4.1)}    \\ 
\hline
\multirow{2}{*}{Distant Tomato}  & Tomato Placement  & 15.6\scriptsize{(5.2)}           & 20.1\scriptsize{(10.6)}           & 23.3\scriptsize{(9.5)}   & \textbf{277.9}\scriptsize{(14.3)}  \\
                             & Tomato Place.\&Delivery & 177.9\scriptsize{(6.1)}          & 180.4\scriptsize{(8.7)}           & 164.8\scriptsize{(19.6)} & \textbf{234.6}\scriptsize{(15.1)}  \\ 
\hline
\multirow{2}{*}{Many Orders}   & Tomato Placement  & 282.6\scriptsize{(16.2)}         & 225.8\scriptsize{(60.8)}          & 259.2\scriptsize{(7.9)}  & \textbf{317.8}\scriptsize{(9.3)}   \\
                             & Tomato Place.\&Delivery & \textbf{329.1}\scriptsize{(5.3)} & \textbf{328.1}\scriptsize{(12.6)} & 295.7\scriptsize{(2.4)}  & \textbf{324.5}\scriptsize{(3.9)}   \\
\bottomrule
\end{tabular}
\vspace{-2mm}
\caption{Average episode reward and standard deviation with unseen testing scripted policies. {\name} significantly outperforms all baselines.}
\vspace{-5mm}
\label{tab:Average-reward-script}
\end{table}
\vspace{-2mm}
\subsection{Cooperation with Human Participants}
\vspace{-1mm}
We recruited 60 volunteers (28.6\% female, 71.4\% male, age between 18–30) by posting the experiment advertisement on a public platform and divided them into 5 groups for 5 layouts. They are provided with a detailed introduction to the basic gameplay and the experiment process. Volunteers are fully aware of all their rights and experiments are approved with the permission of the department. A detailed description of the human study can be found in Appendix \ref{sec:app-human}.
\begin{wrapfigure}{r}{0.28\textwidth}
	\centering
	\vspace{-4mm}
	\hspace{-2mm}
		\includegraphics[width=0.25\textwidth]{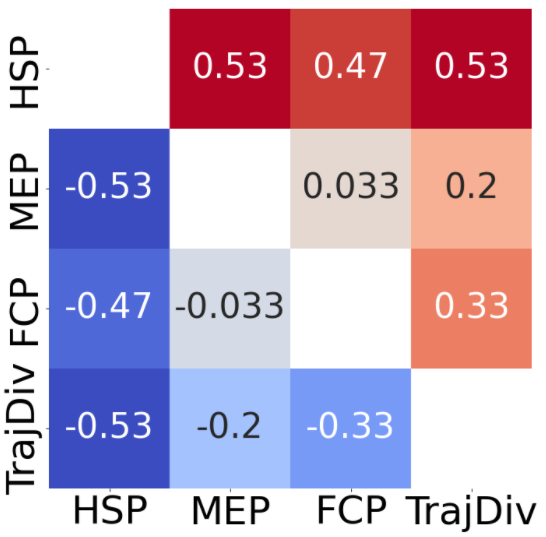}
	\centering 
	\vspace{-2mm}
        \caption{Human preference in the warm-up stage. The unit denotes the difference between the percentage of human players who prefer row partners over column partners and human players who prefer column partners over row partners. {\name} is consistently preferred by human participants with a clear margin.}
        \label{fig:human-rank}
	\vspace{-15mm}
\end{wrapfigure}
The experiment has two stages:
\begin{enumerate}[nolistsep,leftmargin=*]
    \item[$\bullet$] \textbf{Warm-up Stage:} Participants could play the game freely to explore possible AI behaviors. They are asked to rank AI policies according to the degree of assistance during free plays.
    \item[$\bullet$] \textbf{Exploitation Stage:} Participants are instructed to achieve a score as high as possible.
\end{enumerate} 
We note that our user study design differs from that of the original Overcooked paper~\citep{carroll2019utility}.
The additional warm-up stage allows for diverse human behaviors under any possible preference, suggesting a strong testbed for human-assistive AIs.

\subsubsection{Results of the Warm-up Stage}

The warm-up stage is designed to test the performance of AI policies in the face of diverse human preferences. Fig.~\ref{fig:human-rank} visualizes the human preference for different methods reported in the warm-up stage. The unit represents the difference between the percentage of human players who prefer row partners over column partners and human players who prefer column partners over row partners. The detailed calculation method can be found in Appendix \ref{sec:app-humanrank}. {\name} is preferred by humans with a clear margin. Since humans can freely explore any possible behavior, the results in Fig.~\ref{fig:human-rank} imply the strong generalization capability of {\name}. We also summarize feedback from human participants in Appendix \ref{sec:app-humanfeedback}.

\vspace{-2mm}
\subsubsection{Results of the Exploitation Stage} 
The exploitation stage is designed to test the scoring capability of different AIs. Note that it is possible that a human player simply adapts to the AI strategy when instructed to have high scores. So, in addition to final rewards, we also examine the emergent human-AI behaviors to measure the human-AI cooperation level. The experiment layouts can be classified into two categories according to whether the layout allows diverse behavior modes. The first category contains simple \emph{onion-only} layouts that are taken from \citep{carroll2019utility}, including Asymm. Adv., Coord. Ring and Counter Circ.. The second category contains newly introduced layouts with both onions and tomatoes, Distant Tomato and Many Orders, which allow for a much wider range of behavior modes.

\textbf{Onion-only Layouts}:
Fig.~\ref{fig:onion-only} shows the average reward in onion-only layouts for different methods when paired with humans. Among these onion-only layouts, all methods have comparable episode reward in simpler ones (Asymm. Adv. and Coord. Ring), while {\name} is significantly better in the most complex Counter Circ. layout. Fig.~\ref{fig:real-human-random3} shows the frequency of successful onion passing between the human player and the AI player. The learned {\name} policy is able to use the middle counter for passing onions, while the baseline policies are less capable of this strategy.

\textbf{Layouts with Both Onions and Tomatoes:} The results and behavior analysis in Distant Tomato and Many Orders are shown as follows,
\begin{figure}[ht]
\vspace{-13mm}
	\centering
 	\subfloat[\label{fig:onion-only}Onion-Only Layouts]{
 		\includegraphics[height=3.6cm]{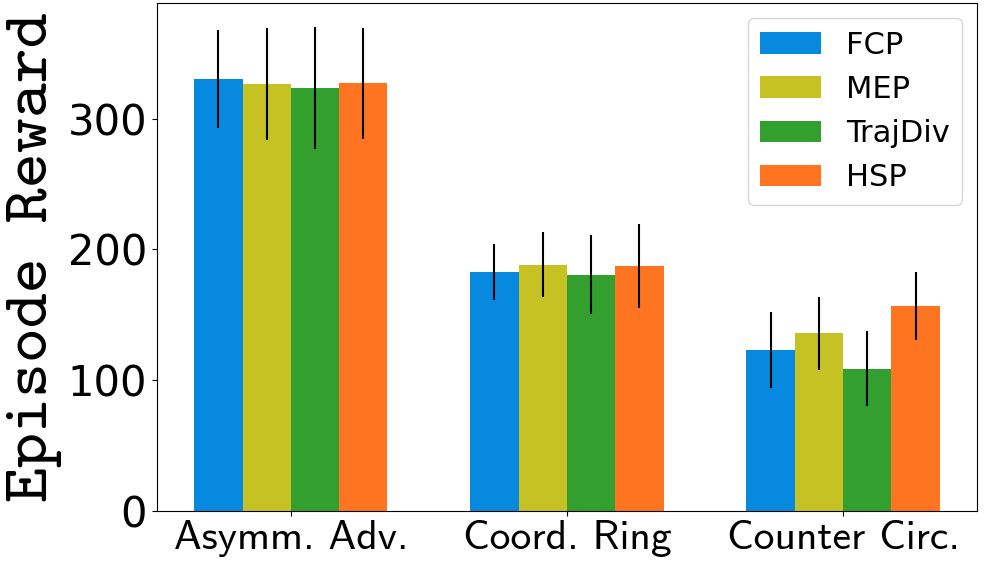}
 		}
    \subfloat[\label{fig:real-human-random3}Counter Circ.]{
		\includegraphics[height=3.6cm]{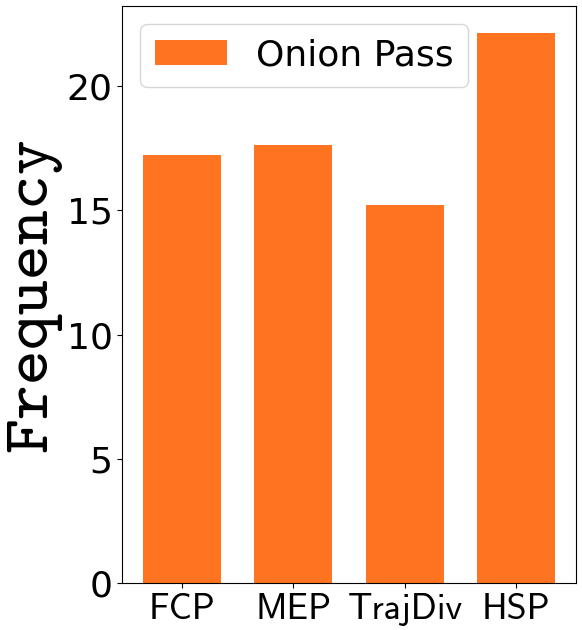}
		}
	\centering 
	\vspace{-2mm}
	\caption{
	(a) Average episode reward in onion-only layouts of different methods when paired with humans in the exploitation stage. {\name} has comparable performance with the baselines in Asymm. Adv. and Coord. Ring, and is significantly better in the most complex Counter Circ. layout. (b) The onion passing frequency in {Counter Circ.} shows that {\name} is the most capable, among other baselines, of passing onions via the counter, suggesting better capabilities to assist humans.}
    \label{fig:real-human}
\vspace{-3mm}
\end{figure}

\begin{enumerate}[nolistsep,leftmargin=*]
    \item[$\bullet$] \textbf{Distant Tomato:} In Distant Tomato, the optimal strategy is always cooking onion soups, while it is suboptimal to cook tomato soups due to the much more time spent on moving. Interestingly, our human-AI experiments found that humans may have diverse biases over onions and tomatoes.
    However, all learned baseline policies tend to have a strong bias towards onions and often place onions into a pot with tomatoes in it already.
    Tab.~\ref{tab:distant-tomato} reports the average number of such \emph{Wrong Placements} made by different AI players. {\name} makes the lowest number of wrong placements and is the only method that can correctly place additional tomatoes into a pot partially filled with tomatoes, labeled \emph{Correct Placements}. This suggests that {\name} is the only effective method to cooperate with biased human strategies, e.g., preferring tomatoes. 
    In addition, as shown in Tab.~\ref{tab:distant-tomato}, even when humans play the optimal strategy of cooking onion soups, {\name} still achieves comparable performance with other methods.

    \item[$\bullet$] \textbf{Many Orders:} In Many Orders, an effective strategy is to utilize all three pots to cook soups. Our experiments found that baseline policies tend to ignore the middle pot. Tab.~\ref{tab:many-orders} shows the average number of soups picked up from the middle pot by different AI players. The learned {\name} policy is much more active in taking soups from the middle pot, leading to more soup deliveries. Furthermore, {\name} achieves a substantially higher episode reward than other methods, as shown in Tab.~\ref{tab:many-orders}.
    
\end{enumerate}

\begin{table}[ht]
\vspace{-2mm}
\centering
\begin{tabular}{ccccc}
\toprule
              & FCP  & MEP  & TrajDiv & HSP  \\ 
\midrule
Onion-Preferred Episode Reward $\uparrow$ & \textbf{343.65} & 325.08 & 334.73 & \textbf{340.3} \\
\hline
Wrong Placements $\downarrow$ & 0.37 & 0.41 & 0.38 & {\textbf{0.21}} \\
Correct Placements $\uparrow$ & 0.0 & 0.0 & 0.0 & {\textbf{1.41}} \\
\bottomrule
\end{tabular}
\vspace{-2mm}
\caption{Average onion-preferred episode reward and frequency of different emergent behaviors in {Distant Tomato} during the exploitation stage. \emph{Onion-Preferred Episode Reward} is the average episode reward when humans prefer onions. \emph{Wrong Placements} and \emph{Correct Placements} are the average numbers of wrong and correct placements into a pot partially filled with tomatoes. 
{\name} makes the lowest number of wrong placements and is the only method that can place tomatoes correctly, suggesting that {\name} is effective at cooperating with biased human strategies.}
\label{tab:distant-tomato}
\end{table}
\begin{table}[ht]
\centering
\vspace{-3mm}
\begin{tabular}{ccccc}
\toprule
              & FCP  & MEP  & TrajDiv & HSP  \\ 
\midrule
Episode Reward $\uparrow$ & 316.81 & 320.61 & 323.52 & {\textbf{382.52}} \\
Number of Soups Picked Up from the Middle Pot $\uparrow$ & 1.93 & 2.03 & 1.33 & {\textbf{5.64}} \\
\bottomrule
\end{tabular}
\vspace{-2mm}
\caption{Average episode reward and average number of picked-up soups from the middle pot by different AI players in Many Orders during the exploitation stage. {\name} achieves significantly better performance and is much more active in taking soups from the middle pot than baselines.}
\label{tab:many-orders}
\vspace{-4mm}
\end{table}

\vspace{-2mm}
\section{Conclusion}
\vspace{-2mm}
We developed \emph{\fullname} (\name) to tackle the problem of zero-shot human-AI cooperation by explicitly modeling human biases as an additional reward function in self-play.
{\name} first generates a pool of diverse strategies and then trains an adaptive policy accordingly.
Experiments verified that agents trained by {\name} are more assistive for humans than baselines in \emph{Overcooked}.
Although our work suggests a new research direction on this fundamentally challenging problem, there are still limitations to be addressed. 
{\name} requires domain knowledge to design a suitable set of events. There exists some work on learning reward functions rather than assuming event-based rewards~\citep{shah2019feasibility, zhou2021continuously}. So a future direction is to utilize learning-based methods to design rewards automatically.
Another major limitation is the computation needed to obtain a diverse policy pool. Possible solutions include fast policy transfer and leveraging a prior distribution of reward functions extracted from human data \citep{NIPS2017_350db081}. Learning and inferring the policy representations of partners could also provide further improvement. 
We leave these issues as our future work.

\clearpage
\subsection*{Acknowledgments}

This research was supported by National Natural Science Foundation of China (No.U19B2019, 62203257, M-0248), Tsinghua University Initiative Scientific Research Program, Tsinghua-Meituan Joint Institute for Digital Life, Beijing National Research Center for Information Science, Technology (BNRist), and Beijing Innovation Center for Future Chips and 2030 Innovation Megaprojects of China (Programme on New Generation Artificial Intelligence) Grant No. 2021AAA0150000.
\bibliography{iclr2023_conference}
\bibliographystyle{iclr2023_conference}

\clearpage
\appendix
We would suggest visiting \url{https://sites.google.com/view/hsp-iclr} for more information.

\section{Theorem Proofs}

For simplicity, we assume state space and action space in our analysis are both discrete and finite, which is exactly the case for \emph{Overcooked}, and the rewards $r$ are bounded: $|r(s,a)|\le R_{\max},\forall s\in\mathcal S,a\in\mathcal A$. 

{\bf Lemma 5.1.}  \emph{\label{lemma1}
    Given an MDP $M=\langle \mathcal S,\mathcal A,P,R_t\rangle$, for any policy $\pi_w:\mathcal S\times \mathcal A\rightarrow [0,1]$, there exists a hidden reward function $R_w$ such that the two-player hidden utility Markov game $M'=\langle\mathcal S,\mathcal A,P,R_w,R_t\rangle$ has a Nash equilibrium $(\pi_a^*,\pi_w^*)$ where $\pi_w^*=\pi_w$.
}

\begin{proof}
Our analysis is based on the maximum entropy reinforcement learning framework~\citep{haarnoja2018soft, ziebart2008maximum, wulfmeier2015maximum}. Given a reward function $R$ and policies of the two players $\pi_1$ and $\pi_2$, we consider following maximum entropy RL objective for policy $\pi_i(1\le i\le 2)$,
$$
J_i(\pi_1,\pi_2|R)=\mathbb E_{\tau}\left[ \sum_t \gamma^t(R(s_t,a_t^{(1)},a_t^{(2)})+\alpha \mathcal H(\pi_i(\cdot|s_t)))\Big|a_t^{(i)}\sim \pi_i(\cdot|s_t)\right]
$$

We shall first constructs $\pi_a$ given policy $\pi_w$ to satisfy $J_2(\pi_w,\pi_a|R_t)\ge J_2(\pi_w,\pi_a'|R_t),\forall \pi_a'$ and secondly constructs $R_w$ such that $J_1(\pi_w,\pi_a|R_w)\ge J_1(\pi_w',\pi_a|R_w),\forall \pi_w'$ is satisfied.

{\bf Step 1: } Construct $\pi_a$ given $\pi_w$. 

Given $\pi_w$, let $\pi_a\in\arg\max_{\pi}J_2(\pi_w,\pi|R_t)$.

{\bf Step 2: } Construct $R_w$ such that $J_1(\pi_w,\pi_a|R_w)\ge J_1(\pi_w',\pi_a|R_w),\forall \pi_w'$ is satisfied given $\pi_w$ and $\pi_a$.

Given a fixed partner $\pi_a$, by regarding $\pi_a$ as part of the environment dynamics, we could consider the dynamics for $\pi_w$ in a single-agent MDP $M'=\langle \mathcal S, \mathcal A, P', R_w,\gamma\rangle$ where $\mathcal S$ is the state space, $\mathcal A$ is the action space, $ P'$ denotes the transition probability and $R_w$ is the reward function to construct. More specifically, $ P'$ is defined as, 
\begin{align}
     P'(s'|s, a) = \sum_{\tilde a}  P(s'|s, a, \tilde a) \cdot \pi_a(\tilde a|s)
\end{align}

In $M'$, given reward $R_w$, the objective of $\pi_w$ becomes,
\begin{align}
    \max_{\pi}\mathbb E_{\tau}\left[\sum_t \gamma^t (R_w(s_t, a_t)+\alpha \mathcal H(\pi(s_t)))\Big| a_t\sim \pi(s_t)\right]
\end{align}

The value function and the Q function could be defined as,
\begin{align}
    V(s)&=\mathbb E_{\tau}\left[\sum_t \gamma^t (R_w(s_t, a_t)+\alpha \mathcal H(\pi_w(s_t)))\Big| a_t\sim \pi_w(s_t),s_0=s\right]\label{eq:V}\\
    &= \sum_a \pi_w(a|s)(R_w(s, a) + \gamma \mathbb E_{s'}[V(s')|s, a]) + \alpha \mathcal H(\pi_w(s))\\
    Q(s,a)&=R_w(s,a) + \gamma \cdot \mathbb E_{s'}[V(s')|s, a]\label{eq:Q}
\end{align}

It is sufficient to construct $R_w$ such that $V(s)$ is a stable point of the \emph{Bellman backup operator}~\citep{sutton2018reinforcement} $\mathcal T^*$ under some $R_w$:
\begin{align}
    (\mathcal T^* V)(s)=\max_{d:\sum_a d(a)=1}\alpha \mathcal H(d) + \sum_a d(a)( R_w(s, a) + \gamma \mathbb E_{s'}[V(s')|s, a])\label{eq:bellman-opt}
\end{align}

Now we assume $V(s)$ is a stable point for Eq.~\ref{eq:bellman-opt} and construct $R_w$. For all $s\in\mathcal S$, $\pi_w(\cdot|s)$ should be a solution to the following maximization problem,

\begin{align}
    \max_{d}&\;\alpha \mathcal H(d) + \sum_a d(a)Q(s, a)\label{prob:maxQ}\\
    s.t.&\;\;\sum_a d(a)=1
\end{align}

Applying KKT conditions over the above optimization problem indicates that,
\begin{align}
\pi_w(\cdot|s)\propto \exp(Q(s, \cdot)/\alpha),\forall s\label{eq:piproptoQ}
\end{align}

Let $\pi^*_w(s)=\arg\max_a \pi_w(a|s), V^*(s)=\max_a Q(s, a), A(s, a)=Q(s,a) - V^*(s)$. By Eq.~\ref{eq:piproptoQ},  we also have
\begin{align}
    A(s, a)=\alpha(\log \pi_w(a|s) - \log\pi_w(\pi_w^*(s)|s))\label{eq:adv}
\end{align}

By definition of value function $V(s)$, 
\begin{align}
    V(s)&=\sum_a \pi_w(a|s)Q(s, a) + \alpha \mathcal H(\pi_w(s))\\
    &=\sum_a \pi_w(a|s)(A(s, a) + V^*(s)) + \alpha \mathcal H(\pi_w(s)) \\
    &=\sum_a \pi_w(a|s)A(s, a) + V^*(s) + \alpha \mathcal H(\pi_w(s)) \\
    &=\sum_a \pi_w(a|s)A(s, a) + R_w(s, \pi_w^*(s)) + \gamma \mathbb E_{s'}[V(s')|s'\sim P'(s, \pi^*_w(s))] + \alpha \mathcal H(\pi_w(s))\\
    &=\mathbb E_{\tau}\left[\sum_t\gamma^t\left(\sum_{a'}\pi_w(a'|s)A(s, a') + R_w(s_t, a_t)+\alpha \mathcal H(\pi_w(s_t))\right)\Big| a_t=\pi^*_w(s_t)\right]
\end{align}

Let $b(s)=R_w(s, \pi^*_w(s))$. Then $V(s)$ is determined given $\pi_w$ and $b$,
\begin{align}\label{eq:value}
    V(s)=\mathbb E_{\tau}\left[\sum_t\gamma^t\left(\sum_{a'}\pi_w(a'|s)A(s, a') + b(s_t)+\alpha \mathcal H(\pi_w(s_t))\right)\Big| a_t=\pi^*_w(s_t)\right]
\end{align}

By $A(s, a)=\alpha(\log \pi_w(a|s) - \log\pi_w(\pi_w^*(s)|s))=Q(s, a) - V^*(s)$, 
\begin{align}
    \alpha(\log \pi_w(a|s) - \log\pi_w(\pi_w^*(s)|s))&=R_w(s, a) + \gamma \mathbb E_{s'}[V(s')|s'\sim P'(s, a)] - V^*(s)\\
   R_w(s, a)&= \alpha\log\left(\frac{\pi_w(a|s)}{\pi_w(\pi_w^*(s)|s)}\right) - \gamma \mathbb E_{s'}[V(s')|s'\sim P'(s, a)] + V^*(s)\label{eq:rw}
\end{align}

To summarize, for policy $\pi_w$, we can construct a valid hidden reward function $R_w$ via following process,
\begin{enumerate}
    \item Choose a function $b:\mathcal S'\rightarrow \mathbb R$.
    \item Compute $A(s,a)$ by Eq.~\ref{eq:adv}.
    \item Compute $V(s)$ and $V^*(s)$ by Eq.~\ref{eq:value}.
    \item Construct $R_w(s, a)$ by $R_w(s,\pi^*_w(s))=b(s)$ and Eq.~\ref{eq:rw}.
\end{enumerate}
\end{proof}

Now we show that, for any $b:\mathcal S'\rightarrow \mathbb R$, under $R_w$ constructed by the above process, $V(s)$ is a stable point of the Bellman backup operator $\mathcal T^*$. This is straightforward. First, constructed $R_w$ ensures that $\alpha(\log \pi_w(a|s) - \log\pi_w(\pi_w^*(s)|s))=Q(s, a) - V^*(s)$ (Eq.~\ref{eq:rw}) and therefore $\pi_w(a|w)\propto \exp(Q(s, a)/\alpha)$, which means $\pi_w$ is a solution for the maximization problem \ref{prob:maxQ}. So
\begin{align}
    (\mathcal T^*V)(s)&=\max_{d}\alpha \mathcal H(d) + \sum_a d(a)( R_w(s, a) + \gamma \mathbb E_{s'}[V(s')]) \\
    &=\sum_a \pi_w(a|s) Q(s, a) + \alpha \mathcal H(\pi_w(s))=V(s).
\end{align}

{\bf Theorem 5.1.}  \emph{\label{theorem1}
   For any $\epsilon >0$, there exists a mapping $\tilde{\pi}_w$ where $\tilde{\pi}_w(R_w)$ denotes the derived policy $\pi_w^*$ in the NE of the hidden utility Markov game $M_w=\langle\mathcal S,\mathcal A, P, R_w, R_t\rangle$ induced by $R_w$,
    and a distribution $P_R:\mathcal R\rightarrow [0, 1]$ over the hidden reward space $\mathcal R$, such that, for any adaptive policy $\pi_A\in \arg\max_{\pi'} \mathbb E_{R_w\sim P_R}[J(\pi', \tilde \pi_w(R_w))]$, $\pi_A$ approximately maximizes the ground-truth objective with at most an $\epsilon$ gap, i.e., $\mathbb E_{\pi_H\sim P_H}[J(\pi_A,\pi_H)]\ge \max_{\pi'}\mathbb E_{\pi_H\sim P_H}[J(\pi',\pi_H)]-\epsilon$.
}

\begin{proof}

Let $K(K>|\mathcal A|)$ be a large positive integer. We construct a discretization of the policy space $\Pi$ by $\Pi_K=\{\pi:\pi(a|s)=\frac{i}{K}\text{ where }i\in [K], \forall s\in\mathcal S,a\in\mathcal A\text{ and }\sum_a \pi(a|s)=1,\forall s\in \mathcal S\}$. Note that $\Pi_K$ is finite, i.e. $|\Pi_K|\le \left(K+1\right)^{|\mathcal S|\cdot|\mathcal A|}$. Let $M=|\Pi_K|$ and $\pi_1,\pi_2,\cdots,\pi_M$ be an ordering of the policies in $\Pi_K$. For simplicity of notation, let $\delta=\frac{|\mathcal A|}{K}$.

Given the discretization $\Pi_K$, it's straightforward to specify the nearest policy $\hat \pi\in \Pi_K$ for any policy $\pi\in \Pi$. Formally, for any policy $\pi\in \Pi$, let $G(\pi)=\arg\min_{i=1,\dots,M}\sum_{s,a}|\pi(a|s)-\pi_i(a|s)|$. An obvious property of $G$ is that, $\forall s\in\mathcal S$, $||\pi(\cdot|s)-G(\pi)(\cdot|s)||_\infty\le \frac{|\mathcal A|}{K}=\delta$.

For two policies $\pi_1$ and $\pi_2$, consider $\pi_1$ playing with $\pi_2$ and $G(\pi_2)$ respectively. Since the action distribution of $\pi_2$ and $G(\pi_2)$ at each state differ at most $\delta$, we have follows,
\begin{align}
    |J(\pi_1,\pi_2)-J(\pi_1, G(\pi_2))|\le\sum_t \gamma^t\cdot(1-\delta)^t\cdot \delta\cdot\frac{2R_{\max}}{1-\gamma}\le \frac{2\delta R_{\max}}{(1-\gamma)^2}\label{neq21}
\end{align}

We can then derive a discretized approximation of the ground-truth policy distribution $P_H$ as follows,
\begin{align}
    \hat P_H(\pi)=\text{Pr}_{\pi'\sim P_H}[\pi=G(\pi')]
\end{align}

We could show that the difference between the objective under the ground-truth policy distribution $P_H$ and that under the approximated policy distribution $\hat P_H$ is bounded. By Eq.~\ref{neq21}, for any adaptive policy $\pi_A$,
\begin{align}
    \big|\mathbb E_{\pi_H\sim \hat P_H}[J(\pi_A,\pi_H)]-\mathbb E_{\pi_H\sim P_H} [J(\pi_A,\pi_H)]\big| &= \big|\mathbb E_{\pi_H\sim P_H}[J(\pi_A,G(\pi_H))-J(\pi_A,\pi_H)]\big| \\
    & \le  \frac{2\delta R_{\max}}{(1-\gamma)^2}
\end{align}

On the other hand, 
consider following an iterative process to find hidden reward functions for policies in $\Pi_K$. For $i=1..M$, we find hidden reward function $R_w^{(i)}$ where $R_w^{(i)}\notin\{R_w^{(j)}|1\le j\le i-1\}$ and $R_w^{(i)}$ could be constructed from $\pi_i$ as in Lemma 5.1. Notice that, by construction rule in Lemma 5.1, such $R_w^{(i)}$ must exists since we can specify arbitrary $b:\mathcal S\rightarrow \mathbb R$. 

Let $\tilde \pi_w(R_w^{(i)})=\pi_i,\forall i=1\dots M$ and the hidden reward distribution $P_R$ be $P_R(R_w^{(i)})=\hat P_H(\pi_i),\forall i=1\cdots M$. We immediately see that, for any adaptive policy $\pi_A$, the objective is equivalent under the approximated policy distribution $\hat P_H$ and hidden reward function distribution $P_R$,
\begin{align}
    \mathbb E_{R_w\sim P_R}[J(\pi_A,\tilde \pi_w(R_w))]=\mathbb E_{\pi_H\sim \hat P_H}[J(\pi_A, \pi_H)]
\end{align}

Finally, for any adaptive policy $\pi_A\in\arg\max_{\pi'}\mathbb E_{R_w\sim P_R}[J(\pi', \tilde \pi_w(R_w))]$ and any policy $\pi'\in \Pi$,

\begin{align}
    \mathbb E_{\pi_H\sim P_H}[J(\pi_A,\pi_H)]&\ge \mathbb E_{\pi_H\sim \hat P_H}[J(\pi_A, \pi_H)]-\frac{2\delta R_{\max}}{(1-\gamma)^2}\\
    &=\mathbb E_{R_w\sim P_R}[J(\pi_A,\tilde \pi_w(R_w))]-\frac{2\delta R_{\max}}{(1-\gamma)^2}\\
    &\ge \mathbb E_{R_w\sim P_R}[J(\pi',\tilde \pi_w(R_w))]-\frac{2\delta R_{\max}}{(1-\gamma)^2}\\
    &=\mathbb E_{\pi_H\sim \hat P_H}[J(\pi', \pi_H)]-\frac{2\delta R_{\max}}{(1-\gamma)^2}\\
    &\ge \mathbb E_{\pi_H\sim P_H}[J(\pi', \pi_H)]-\frac{4\delta R_{\max}}{(1-\gamma)^2}
\end{align}

Let $K\ge\frac{4|\mathcal A|R_{\max}}{\epsilon(1-\gamma)^2}$ and we have $\mathbb E_{\pi_H\sim P_H}[J(\pi_A,\pi_H)]\ge \max_{\pi'}\mathbb E_{\pi_H\sim P_H}[J(\pi',\pi_H)]-\epsilon$.

\end{proof}

\section{Environment Details}

\begin{figure}[htbp]
    \centering
    \includegraphics[width=\textwidth]{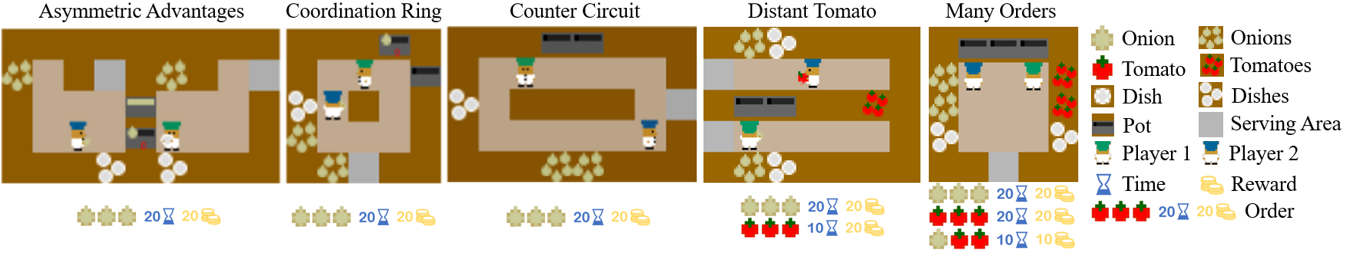}
    \caption{All 5 layouts used in our work  (from left to right): \emph{Asymmetric Advantage}, \emph{Coordination Ring}, \emph{Counter Circuit}, \emph{Distant Tomato}, and \emph{Many Orders}, each featuring specific cooperation patterns we want to study.}
    \label{fig:all_layouts}
\end{figure}

\subsection{Description}

The Overcooked Environment, first introduced in \citep{carroll2019utility}, is based on the popular video game Overcooked where multiple players cooperate to finish as many orders as possible within a time limit. In this simplified version of the original game, two chiefs, each controlled by a player (either human or AI), work in grid-like \emph{layouts}. Chiefs can move between non-table tiles and interact with table tiles by picking up or placing objects. Ingredients (e.g., onions and tomatoes) and empty dishes can be picked up from the corresponding dispenser tiles and placed on empty table tiles or into the pots. The typical pipeline for completing an order is (1) players put appropriate ingredients into a pot; (2) a pot starts cooking automatically once filled and takes a certain amount of time (depending on the recipe) to finish; (3) a player harvests the cooked soup with an empty dish and deliver it to the serving area.

The \textbf{observation} for an agent includes the whole layout, items on the counter and pots, player positions, orders, and time. The possible \textbf{actions} are up, down, left, right, no-op, and "interacting" with the tile the player is facing. \textbf{Reward} is given to both agents upon successful soup delivery, with the amount varying with the type of soup. An episode of the game terminates when the time limit is reached.

The environment used in \citep{carroll2019utility} has only onions as ingredients and onion soups as orders. In our work, we evaluate all methods in three of them, namely \emph{Asymmetric Advantage}, \emph{Coordination Ring}, and \emph{Counter Circuit}, each designed to enforce a specific cooperation pattern. 

Our work introduces two new layouts: \emph{Distant Tomato} and \emph{Many orders}, with new ingredients and order types to make cooperation more challenging. In \emph{Distant Tomato}, a dish of onion soup takes 20 ticks to finish and gives 20 rewards when delivered, while a tomato soup takes 10 ticks and gives the same reward but needs more movements to get the ingredient. The two players need to agree on which type of soup to cook in order to reach a high score. Failure in cooperation may result in tomato-onion soups that give no reward. In many orders, there are three types of orders: onion, tomato, and 1-onion-2-tomato. To fully utilize the three pots, the players need to work seamlessly in filling not just the pots near each of them but also the pot in the middle.

We show all the layouts in Fig.\ref{fig:all_layouts}.  and conclude the cooperation pattern of our interest as follows.
\begin{itemize}
    \item \emph{Asymmetric Advantage} tests whether the players can choose a strategy to their strengths.
    \item \emph{Coordination Ring} requires the players not to block each other when traveling between the two corners.
    \item \emph{Counter Circuit} embeds a non-trivial but efficient strategy of passing onions through the middle counter, which needs close cooperation.
    \item \emph{Distant Tomato} and \emph{Many Orders} both encourage the players to reach an agreement on the fly in order to achieve a high reward.
\end{itemize}

\subsection{Events}
\label{sec:events}

In Overcooked, we consider the following events for random search in {\name} and reward shaping during training of all methods:
\begin{itemize}
    \item putting an onion/tomato/dish/soup on the counter,
    \item picking up an onion/tomato/dish/soup from the counter,
    \item picking up an onion from the onion dispenser,
    \item picking up a tomato from tomato dispenser,
    \item picking up a dish from the dish dispenser,
    \item picking up a ready soup from the pot with a dish,
    \item placing an onion/tomato into the pot,
    \item valid placement: after the placement, we can finish an order with a positive reward by placing other ingredients,
    \item optimal placement: the placement is optimal if the maximum order reward we can achieve for this particular pot is not decreased after the placement,
    \item catastrophic placement: the placement is catastrophic if the maximum order reward we can achieve for this particular pot decreases from positive to zero after the placement,
    \item useless placement: the placement is useless if the maximum order reward we can achieve for this particular pot is already zero before the placement,
    \item useful dish pickup: picking up a dish is useful when there are no dishes on the counter, and the number of dishes already taken by players is less than the total number of unready and ready soups,
    \item delivering a soup to the serving area.
\end{itemize}

Additionally, in \emph{Distant Tomato}, we consider the following events only for reward shaping,
\begin{itemize}
    \item placing a tomato into an empty pot,
    \item optimal tomato placement: the placement is optimal and a tomato placement,
    \item useful tomato pickup: the agent picks up a tomato when the partner isn't holding a tomato, and there is a pot that is not full but only has tomatoes in it.
\end{itemize}

\section{Overcooked Version}

In our experiments, we use two versions of Overcooked for a fair comparison with prior works and introduce challenging layouts. One version, in which we tested \emph{Asymmetric Advantages}, \emph{Coordination Ring} and \emph{Counter Circuit}, is consistent with the "neurips2019" branch in the released GitHub repository of \citep{carroll2019utility}. We remark that MEP~\citep{zhao2021maximum} also follows this version. Following this also allows us to perform an evaluation with human proxy models provided in the released code of \citep{carroll2019utility}. The other version is an up-to-date version of Overcooked, which supports tomatoes and user-defined orders. We notice that a pot automatically starts cooking soup once there are three items in it in the former version, while it requires an additional "interact" action to start cooking in the latter version. This additional "interact" is required in the latter version since it supports orders with different amounts of ingredients. However, having an additional "interact" significantly influences a human player's interactive experience. Therefore, we make modifications on the latter version to restrict orders to $3$ items and support auto-cooking when there are $3$ items. For more details, please refer to the released code.

\section{Implementation Details}
\label{sec:app-implement}
\subsection{{\name}}

\begin{algorithm}[H]
\caption{Hidden-Utility Self-Play}\label{algo:{\name}}
\For{$i=1\rightarrow N$}{
    Train $\pi_w^{(i)}$ and $\pi_a^{(i)}$ under sampled $R_w^{(i)}$\;
}
Run greedy policy selection to only keep $K$ policies\;
Initial policy $\pi_A$\;
\Repeat{enough iterations}{
    Rollout with $\pi_A$ and sampled $\pi_w^{(i)}$\;
    Update $\pi_A$\;
}
\end{algorithm}

The pseudocode of {\name} is shown in Algo.~\ref{algo:{\name}}. We implemented {\name} on top of MAPPO~\citep{yu2021surprising}. Following the standard practice, we use multiprocessing to collect trajectories in parallel and then update the models. 
In the first stage, we use MLP policies, which empirically yield better results. 
In the second stage, we use RNN policies so that the adaptive policy could infer the intention of its partner by observing the history of its partner and make decisions accordingly for better adaptation.
As suggested in\citep{tang2020discovering}, we add the identities of the policies in the policy pool as an additional feature to the critic. For better utilization of the computation resources, each environment sub-process loads a uniformly sampled policy and performs inference on CPUs, while the inference of the adaptive policy is batched across sub-processes in a GPU. 

\subsection{Baselines}
\label{sec:app-baselines}
For a fair comparison, we implement all baselines to be two-staged and train layout-specific agents.

We remark that our implementation of MEP achieves substantially higher scores than reported in the original paper~\citep{zhao2021maximum} when evaluated with the same human proxy models as MEP. All baselines are implemented with techniques stated above: loading policies from the pool per sub-process and the additional feature of identities of policies in the policy pool. We detail the baselines here and point out the difference with the original papers,

\textbf{FCP\citep{strouse2021collaborating}:} We list the differences between our implementation and the original FCP as follows,
\begin{enumerate}
    \item The original FCP uses image-based egocentric observations, while we use feature-based observations as provided in Overcooked.
    \item The original FCP uses a pool size of $96$ while we use $36$. We empirically found $36$ a sufficiently large pool size in our experiments. As shown in Table~\ref{tab:Average-reward-poolsize}, in the three layouts that have human proxy models, there is no significant difference between using a pool size of $36$ and of $72$.
\end{enumerate}

\textbf{MEP\citep{zhao2021maximum}:} We list the differences between our implementation and the original MEP as follows,
\begin{enumerate}
    \item While the released code of MEP uses MLP policy in the second training stage, we found RNN policy to work better. Intuitively, for better cooperation, the adaptive policy should infer the intention of its partner by observing the state-action history. 
    \item MEP uses a pool size of $15$ while we use $36$.
    \item MEP uses prioritized sampling in the second stage, which favors weak policies in the pool, while we adopt uniform sampling for MEP since we found prioritized sampling not helpful with our carefully tuned implementation (shown in Table \ref{tab:Average-reward-mep}). 
    \item In the released code of MEP, the policy updates are performed on data against only one policy from the pool, while we perform policy updates on data against many policies from the pool. This avoids the update from being biased towards some specific policies.
\end{enumerate}

\begin{table}[htbp]
\centering
\begin{tabular}{ccccc}
\toprule
                                      & Pos. & Asy. Adv. & Coor. Ring & Coun. Circ.  \\
                                      \midrule
\multirow{2}{*}{Uniform Sampling}     & 1    & 291.7\scriptsize{(4.6)} & 161.8\scriptsize{(0.7)}  & 108.8\scriptsize{(4.2)}     \\
                                      & 2    & 203.4\scriptsize{(2.0)} & 164.2\scriptsize{(2.1)}  & 111.1\scriptsize{(0.7)}     \\
                                      \hline
\multirow{2}{*}{Prioritized Sampling} & 1    & 284.6\scriptsize{(3.2)} & 161.2\scriptsize{(1.4)}  & 94.4\scriptsize{(2.3)}      \\
                                      & 2    & 218.8\scriptsize{(2.4)} & 167\scriptsize{(4.5)}    & 99.8\scriptsize{(1.8)}      \\
\bottomrule
\end{tabular}
\caption{Average episode reward and standard deviation (over 5 seeds) with different sampling methods of MEP. The "1" and "2" indicates the roles played by AI policies.}
\label{tab:Average-reward-mep}
\vspace{-5mm}
\end{table}

\textbf{TrajDiv\citep{lupu2021trajectory}:} While the original TrajDiv is tested in hand-crafted MDPs and Hanabi, we test TrajDiv in Overcooked. Although \citep{lupu2021trajectory} suggests training the adaptive policy and the policy pool together in a single stage, we choose to follow MEP and FCP to have a two-staged design that trains the adaptive policy in the second stage.

\subsection{Scripted Policies}
\label{sec:app-script}
To evaluate all methods with policies that have strong preferences, we consider the following scripted policies,

\begin{itemize}
    \item \emph{Onion/Tomato/Dish Everywhere} continuously tries to put onions, tomatoes or dishes over the counter.
    \item \emph{Onion/Tomato Placement} always tries to put onion or tomato into the pot.
    \item \emph{Delivery} delivers a ready soup to the serving area whenever possible.
    \item \emph{Onion/Tomato Placement and Delivery} puts tomatoes/onions into the pot in half of the time and tries to deliver soup in the other half of the time.
\end{itemize}

For \emph{Counter Circuit}, we additionally consider a scripted policy, named \emph{Onion to Middle Counter}, which keeps putting onions randomly over the counter in the middle of the layout. 

Input to these scripted policies is the ground-truth state of the game, which is accessible via the game simulator. When a scripted policy is unable to finish the event of its interest at some state, the scripted policy would walk to a random empty grid. For example, \emph{Onion Placement} would choose a random walk when all pots are full. We ensure that these scripted policies are strictly different from policies in the policy pool of {\name}. For more details, please refer to the released code.

\begin{table}
\centering
\begin{tabular}{ccccc}
\toprule
              & Biased  & Scripted  \\ 
\midrule
Asymm. Adv. & 0.56 & {\textbf{0.85}}   \\
Coord. Ring & 0.59 & {\textbf{0.72}} \\ 
Counter Circ. & 0.56 & {\textbf{0.73}} \\
Dist. Tomato & 0.70 & {\textbf{1.90}} \\
Many Orders & 0.55 & {\textbf{1.00}} \\
\bottomrule
\end{tabular}
\caption{The average event-based difference of biased and scripted policies respectively.}
\label{tab:script-diff}
\end{table}

We also provide evidence to show scripted policies are sufficiently different from those in the training pool. We use the expected event count of scripted and biased policies to support our claim. Recall that expected event count for a pair of policy $\pi_a,\pi_b$ is $\text{EC}(\pi_a,\pi_b)=\mathbb E[\sum_{t=1}^T \phi(s_t,a_t)|\pi_a,\pi_b]$. Let $\pi_{HSP}$ be the {\name} adaptive policy, $\{\pi_w^{(n)}\}_{n\in[N]}$ be the set of biased policies in the training pool, and $\{\pi_s^{(m)}\}_{m\in M}$ be the set of scripted policies. For convenience, let $\Pi=\{\pi_w^{(n)}\}_{n\in[N]}\cup \{\pi_s^{(m)}\}_{m\in M}$ be the \textbf{union} of biased policies and scripted policies. For each policy $\pi' \in \Pi$, we measure how close it is to the rest of policies in $\Pi$ in the expected event count, i.e. the \emph{event-based difference} $\text{EventDiff}_\Pi(\pi')=\min_{\pi''\in \Pi\backslash \{\pi'\}} \sum_k c_k\cdot|\text{EC}_k(\pi',\pi_{HSP}) - \text{EC}_k(\pi'', \pi_{HSP})|$ where $c_k$ is a frequency normalization constant. Then a large event-based difference indicates that $\pi'$ is sufficiently different from other policies in $\Pi$. We calculate the event-based difference for all biased and scripted policies. Table.~\ref{tab:script-diff} reports the average event-based difference between biased and scripted policies, respectively. Scripted policies consistently have a larger average event-based difference, indicating scripted policies are sufficiently different from biased policies, which are used for training the {\name} adaptive policy.

\section{Training Details}
\label{sec:app-training}
\subsection{Hyperparameters}
\label{sec:app-hyperparameters}
{\name} and baselines are all two-staged solutions by first constructing a policy pool and then training an adaptive policy $\pi_A$ to maximize the game reward w.r.t. the induced pool. 

The network architecture in both two stages is composed of 3 convolution layers with max pooling. Hyperparameters of these layers are listed in Table \ref{tab:cnn}. Each layer is followed by a max pooling layer with a kernel size of $2$. For MLP policies, we add two linear layers after the convolution. For RNN policies, we add a 1-layer GRU after the convolution and two linear layers after the GRU layer. The hidden sizes for these linear layers and the GRU layer are all $64$. We use ReLU as the activation function between layers and LayerNorm after GRU and linear layers except the last one.
The output is a 6-dim vector denoting the categorical action distribution. 

Common hyperparameters for all methods in 5 layouts are listed in Table \ref{tab:hyperparameter1} and Table \ref{tab:hyperparameter2}. 
Specifically, for MEP, we use the suggested hyperparameters from the original paper~\citep{zhao2021maximum}. Detailed hyperparameters of MEP are shown in Table \ref{tab:hyperparameter-mep}, where population entropy coef. adjusts the importance of the population entropy term. 
Detailed hyperparameters of TrajDiv are shown in Table \ref{tab:hyperparameter-traj}, where traj. gamma is the discounting factor used in local action kernel and diversity coef. adjusts the importance of the diversity term. For each one of MEP, FCP and TrajDiv, we train $12$ policies in the first stage and, following the convention of MEP~\citep{zhao2021maximum} and FCP~\citep{strouse2021collaborating}, take the init/middle/final checkpoints for each policy to build up the policy pool, leading to a pool size of $36$. For {\name}, we use a random search to first train $36$ biased policies and then filter out $18$ biased policies from them. We then combine these biased policies and past checkpoints of $6$ policies in the policy pool of MEP to build up the policy pool of {\name}, again leading to a pool size of $36$.

\begin{table}[h]
\centering
\begin{tabular}{ccccc}
\toprule
Layer &                      Out Channels &         Kernel Size & Stride & Padding \\
\midrule
1 & 32 & 3 & 1 & 1 \\
    2 &  64 & 3 & 1 & 1 \\
3 & 32 & 3 & 1 & 1 \\
\bottomrule
\end{tabular}
\caption{CNN feature extractor hyperparameters.}
\label{tab:cnn}
\vspace{-5mm}
\end{table}

\begin{table}[h]
\centering
\begin{tabular}{cc}
\toprule
common hyperparameters      & value  \\
\midrule
entropy coef. & 0.01 \\
gradient clip norm          & 10.0 \\
GAE lambda                   & 0.95    \\              
gamma                      & 0.99 \\
value loss                        & huber loss \\
huber delta                 & 10.0   \\
mini batch size           & batch size {/} mini-batch  \\
optimizer                & Adam       \\
optimizer epsilon            & 1e-5       \\
weight decay             & 0          \\
network initialization  & Orthogonal \\
use reward normalization   &True \\
use feature normalization   &True \\
learning rate & 5e-4 \\
parallel environment threads & 100 \\
ppo epoch & 15\\
environment steps & 10M\\
episode length & 400\\
reward shaping horizon & 100M\\
\bottomrule
\end{tabular}
\caption{Common hyperparameters in the first stage.}
\label{tab:hyperparameter1}
\end{table}

\subsection{Constructing the Policy Pool for {\name}}

To construct the policy pool for {\name}, we perform a random search over possible hidden reward functions. Each reward function is formulated as a linear function over the event-based features, i.e. $\mathcal R=\{R_w:R_w(s, a_1, a_2)=\phi(s, a_1, a_2)^Tw, ||w||_{\infty}\le C_{\max}\}$ where $\phi:\mathcal S\times \mathcal A\times\mathcal A\rightarrow \mathbb R^{m}$ specifies occurrences of different events when taking joint action $(a_1,a_2)$ at state $s$. To perform random search, instead of directly sampling each $w_j$ from the section $[-C_{\max}, C_{\max}]$, we sample each $w_j$ from a set of possible values $C_j$. We detail the $C_j$ for each event on each layout here. Tab.~\ref{tab:random-search-old-C} shows $C_j$ in \emph{Asymmetric Advantages}, \emph{Coordination Ring} and \emph{Counter Circuit}. Tab.~\ref{tab:random-search-far_tomato-C} and Tab.~\ref{tab:random-search-many_orders-C} show $C_j$ in \emph{Distant Tomato} and \emph{Many Orders} respectively. A detailed description of the events is shown in Sec.~\ref{sec:events}. Note that in addition to events, we also include order reward as one element in a random search. 

To filter out duplicated policies, we define an \emph{event-based diversity} for a subset $S$, i.e. $\text{ED}(S)=\sum_{i,j\in S}\sum_k c_k\cdot |\text{EC}_k^{(i)} - \text{EC}_k^{(j)}|$ where $\text{EC}_{k}^{i}$ is the expected number of occurrences of event type $k$ for biased policy $\pi_w^{(i)}$. The coefficient $c_k$ balances the importance of different kinds of events. We simply set $c_k$ as a normalization constant, i.e. $c_k=\left(\max_{i\in[N]}\text{EC}_k^{(i)}\right)^{-1}$.

\begin{table}[ht]
\centering
\begin{tabular}{cc}
\toprule
common hyperparameters      & value  \\
\midrule
entropy coef. & 0.01 \\
gradient clip norm          & 10.0 \\
GAE lambda                   & 0.95    \\              
gamma                      & 0.99 \\
value loss                        & huber loss \\
huber delta                 & 10.0   \\
mini batch size           & batch size {/} mini-batch  \\
optimizer                & Adam       \\
optimizer epsilon            & 1e-5       \\
weight decay             & 0          \\
network initialization  & Orthogonal \\
use reward normalization   &True \\
use feature normalization   &True \\
learning rate & 5e-4 \\
parallel environment threads & 300 \\
ppo epoch & 15\\
environment steps & 100M\\
episode length & 400\\
reward shaping horizon & 100M\\
policy pool size & 36 \\
\bottomrule
\end{tabular}
\caption{Common hyperparameters in the second stage.}
\label{tab:hyperparameter2}
\end{table}

\begin{table}[ht]
\centering
\begin{tabular}{cc}
\toprule
hyperparameters      & value  \\
\midrule
population entropy coef.          & 0.01 \\
\bottomrule
\end{tabular}
\caption{MEP hyperparameters in the first stage.}
\label{tab:hyperparameter-mep}
\end{table}

\begin{table}[ht]
\centering
\begin{tabular}{cc}
\toprule
hyperparameters      & value  \\
\midrule
traj. gamma          & 0.5 \\
diversity coef.         & 0.1 \\
\bottomrule
\end{tabular}
\caption{TrajDiv hyperparameters in the first stage.}
\label{tab:hyperparameter-traj}
\end{table}

\begin{table}[ht]
\centering
\begin{tabular}{cc}
\toprule
Event     & $C_j$  \\
\midrule
Picking up an onion from onion dispenser & -10, 0, 10 \\
Picking up a dish from dish dispenser & 0, 10 \\
Picking up a ready soup from the pot& -10, 0, 10\\
Placing an onion into the pot & -10, 0, 10\\
Delivery & -10, 0 \\
Order reward & 0, 1 \\
\bottomrule
\end{tabular}
\caption{$C_j$ for random search in Asymmetric Advantages, Coordination Ring and Counter Circuit.}
\label{tab:random-search-old-C}
\end{table}

\begin{table}[ht]
\centering
\begin{tabular}{cc}
\toprule
Event     & $C_j$  \\
\midrule
Picking up an onion from onion dispenser & -5, 0, 5 \\
Picking up a tomato from tomato dispenser & 0, 10, 20 \\
Picking up a dish from dish dispenser & 0, 10 \\
Picking up a soup & -5, 0, 5\\
Viable placement & -10, 0, 10 \\
Optimal placement & -10, 0, 10\\
Catastrophic placement & 0, 10 \\
Placing an onion into the pot & -10, 0, 10\\
Placing a tomato into the pot & -10, 0, 10\\
Delivery & -10, 0 \\
Order reward & 0, 1 \\
\bottomrule
\end{tabular}
\caption{$C_j$ for random search in Distant Tomato.}
\label{tab:random-search-far_tomato-C}
\end{table}

\begin{table}[ht]
\centering
\begin{tabular}{cc}
\toprule
Event     & $C_j$  \\
\midrule
Picking up an onion from onion dispenser & -5, 0, 5 \\
Picking up a tomato from tomato dispenser & 0, 10, 20 \\
Picking up a dish from dish dispenser & 0, 5 \\
Picking up a soup & -5, 0, 5\\
Viable placement & -10, 0, 10 \\
Optimal placement & -10, 0\\
Catastrophic placement & 0, 10 \\
Placing an onion into the pot & -3, 0, 3\\
Placing a tomato into the pot & -3, 0, 3\\
Delivery & -10, 0 \\
Order reward & 0, 1 \\
\bottomrule
\end{tabular}
\caption{$C_j$ for random search in Many Orders.}
\label{tab:random-search-many_orders-C}
\end{table}

\begin{table}[ht]
\centering
\begin{tabular}{cc}
\toprule
Event     & Value  \\
\midrule
Optimal placement & 3\\
Picking up a dish from dish dispenser & 3 \\
Picking up a ready soup from the pot & 5 \\
\bottomrule
\end{tabular}
\caption{Reward shaping for Asymmetric Advantages, Coordination Ring and Counter Circuit in the first stage.}
\label{tab:rew-shaping-first-old}
\vspace{-5mm}
\end{table}

\begin{table}[ht]
\centering
\begin{tabular}{cc}
\toprule
Event     & Value  \\
\midrule
Optimal placement & 3\\
Picking up a dish from dish dispenser & 3 \\
Picking up a ready soup from the pot & 5 \\
\bottomrule
\end{tabular}
\caption{Reward shaping for Asymmetric Advantages, Coordination Ring and Counter Circuit in the second stage.}
\label{tab:rew-shaping-second-old}
\end{table}

\begin{table}[ht]
\centering
\begin{tabular}{cc}
\toprule
Event     & Value  \\
\midrule
Picking up a dish from dish dispenser & 3 \\
Picking up a ready soup from the pot & 5 \\
\bottomrule
\end{tabular}
\caption{Reward shaping for Distant Tomato and Many Orders in the first stage.}
\label{tab:rew-shaping-first-new}
\end{table}

\begin{table}[ht]
\centering
\begin{tabular}{cc}
\toprule
Event     & Value  \\
\midrule
Picking up a dish from the dish dispenser & 3 \\
Picking up a ready soup from the pot & 5 \\
\bottomrule
\end{tabular}
\caption{Reward shaping for Many Orders in the second stage.}
\label{tab:rew-shaping-second-many-orders}
\end{table}

\begin{table}[ht]
\centering
\begin{tabular}{cc}
\toprule
Event     & Value  \\
\midrule
Picking up a dish from the dish dispenser & 3 \\
Picking up a ready soup from the pot & 5 \\
Useful tomato pickup & 10 \\
Optimal tomato placement & 5 \\
Placing a tomato into an empty pot & -15 \\ 
\bottomrule
\end{tabular}
\caption{Reward shaping for Distant Tomato in the second stage.}
\label{tab:rew-shaping-second-distant-tomato}
\end{table}

\subsection{Reward Shaping}

We use reward shaping during training in all layouts, detailed as follows,
\begin{itemize}
    \item In the first stage, the reward shaping for \emph{Asymmetric Advantages}, \emph{Coordination Ring} and \emph{Counter Circuit} is shown in Table.~\ref{tab:rew-shaping-first-old} and that for \emph{Distant Tomato} and \emph{Many Orders} is shown in Table.~\ref{tab:rew-shaping-first-new}. Note that we do not use reward shaping when training biased policies for {\name} in the first stage.
    \item In the second stage, the reward shaping for \emph{Asymmetric Advantages}, \emph{Coordination Ring} and \emph{Counter Circuit} is shown in Table.~\ref{tab:rew-shaping-second-old}. Reward shaping for \emph{Many Orders} is shown in Table.~\ref{tab:rew-shaping-second-many-orders} and that for \emph{Distant Tomato} is shown in Table.~\ref{tab:rew-shaping-second-distant-tomato}. The factor of shaped reward anneals from $1$ to $0$ during the whole course of training in all layouts except \emph{Distant Tomato}, in which the factor anneals from $1$ to $0.5$. 
\end{itemize}

\section{Full Results}
\label{sec:app-results}
\subsection{Cooperation with Learned Human Models}
\label{sec:app-results-learned}
Table~\ref{tab:Average-reward-proxy} shows average episode reward and standard deviation (over 5 seeds) on 3 layouts for different methods played with human proxy policies. All values within 5 standard deviations of the maximum episode return are marked in bold. 
These three simple layouts may not fully reflect the performance gap between the baselines and {\name}. The results with learned human models are reported for a fair comparison with existing SOTA methods. Besides, our implementation of the baselines achieves substantially better results than their original papers with the same human proxy models, making the improvement margin look smaller. We also remark that the learned human models have limited representation power to imitate natural human behaviors that typically cover many behavior modalities. Here we give empirical evidence of the learned human models failing to fully reflect human behaviors.

\subsubsection{Empirical Evidence}
\label{sec:app-evidence}

The original Overcooked paper~\citep{carroll2019utility} collected human-play trajectories. We then collect game trajectories played by the learned human models and compare them with human-play trajectories by measuring \emph{self-delivery ratio}, i.e., the ratio of deliveries by the specific player to the total delivery number in a trajectory, and \emph{self-cooking ratio}, which is the ratio of onions that the player places in the pot to the total pot placement number in a trajectory. The distributions of these trajectories are demonstrated in Fig.~\ref{fig:human-imitate-traj}. From the figure, we can observe that the learned human models can not fully cover human behaviors. This suggests that evaluation results with the learned human models can not provide a comprehensive comparison among different methods.

\begin{figure*}[t!]
	\centering
    \subfloat[Asymmetric Advantages]{
		\includegraphics[width=0.35\textwidth]{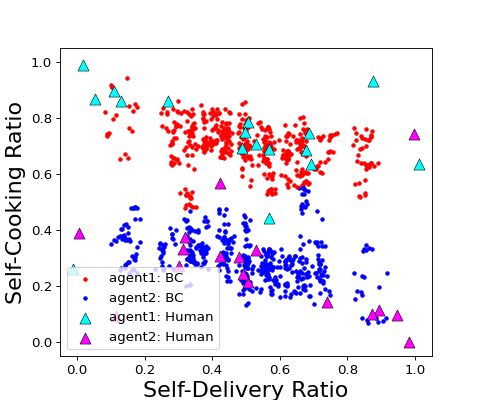}
		}
    \subfloat[Coordination Ring]{
		\includegraphics[width=0.35\textwidth]{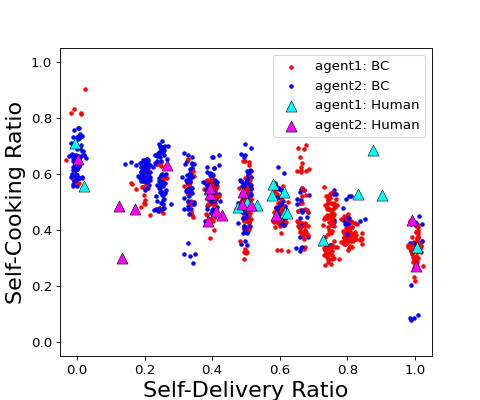}
		}
    \subfloat[Counter Circuit]{
		\includegraphics[width=0.35\textwidth]{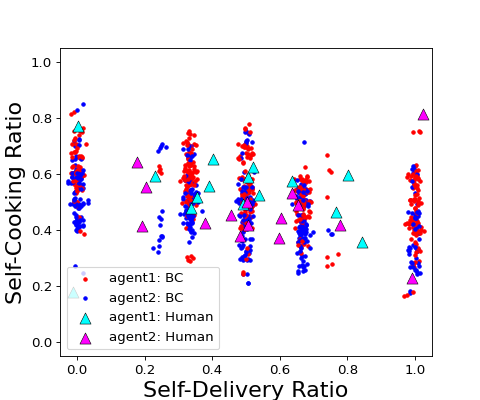}
		}
	\centering 
	\caption{Trajectories induced by the learned human models and human players in Asymmetric Advantages, Coordination Ring and Counter Circuit. Each point or triangle denotes a trajectory with the X-axis coordinate being the self-cooking ratio, which is the ratio of onions the player places in the pot to the total amount of placements in the trajectory, and the Y-axis coordinate being the self-delivery ratio, which is the ratio of deliveries given by the player to the total number of deliveries in the trajectory. Triangles and points denote trajectories induced by human players and learned human models, respectively. Different colors stand for different player indices. "BC" represents the learned human models, and "Human" denotes human players. Clearly, trajectories induced by the learned human models can not fully cover those by human players.}
\label{fig:human-imitate-traj}
\vspace{-3mm}
\end{figure*}

\begin{table}[hp]
\centering
\begin{tabular}{ccccc} 
\toprule
                                     & Pos. & Asy. Adv.           & Coor. R.          & Coun. Circ.          \\ 
\midrule
\multirow{2}{*}{FCP}     & 1    & 282.8\scriptsize{(9.4)}          & \textbf{161.3}\scriptsize{(1.6)}          & 95.9\scriptsize{(2.0)}            \\
                                     & 2    & 203.8\scriptsize{(8.2)}          & \textbf{161.0}\scriptsize{(2.7)}          & 92.7\scriptsize{(1.3)}            \\ 
\hline
\multirow{2}{*}{MEP}     & 1    & 291.7\scriptsize{(4.6)}          & \textbf{161.8}\scriptsize{(0.7)} & \textbf{108.8}\scriptsize{(4.2)}              \\
                                     & 2    & 203.4\scriptsize{(2.0)}          & \textbf{164.2}\scriptsize{(2.1)} & \textbf{111.1}\scriptsize{(0.7)}              \\ 
\hline
\multirow{2}{*}{TrajDiv} & 1    & 289.3\scriptsize{(8.8)}          & 150.8\scriptsize{(3.1)}          & 60.1\scriptsize{(5.0)}              \\
                                     & 2    & 194.2\scriptsize{(0.7)}          & 142.1\scriptsize{(2.3)}          & 53.7\scriptsize{(12.4)}              \\ 
\hline
\multirow{2}{*}{HSP}     & 1    & \textbf{300.3}\scriptsize{(2.2)} & \textbf{160.0}\scriptsize{(2.6)}          & \textbf{107.4}\scriptsize{(3.5)}  \\
                                     & 2    & \textbf{217.1}\scriptsize{(3.3)} & \textbf{160.6}\scriptsize{(3.3)}          & \textbf{106.6}\scriptsize{(3.0)}           \\
\bottomrule
\end{tabular}
\caption{Average episode reward and standard deviation (over 5 seeds) on 3 layouts for different methods played with human proxy policies. All values within 5 standard deviations of the maximum episode return are marked in bold. The Pos. column indicates the roles played by AI policies.}
\label{tab:Average-reward-proxy}
\end{table}

\subsection{Ablation Studies}

\subsubsection{Pool Size}
\label{sec:app-poolsize}
Table~\ref{tab:Average-reward-poolsize} shows the average episode reward on 3 layouts with different sizes of the final policy pool for training the adaptive policy. Since increasing the pool size to 72 gives little improvement as suggested by the result, we use 36 in our experiments for computation efficiency.

\begin{table}[h]
\centering
\begin{tabular}{ccccc|ccc} 
\toprule
                         & Pos.                 & Asy. Adv.  & Coor. R.   & Coun. Circ.        & Asy. Adv.   & Coor. R.   & Coun. Circ.       \\ 
\hline
\multicolumn{1}{l}{}     & \multicolumn{1}{l}{} & \multicolumn{3}{c|}{policy pool size = 36} & \multicolumn{3}{c}{policy pool size = 72}  \\ 
\midrule
\multirow{2}{*}{FCP}     & 1                    & 282.8\scriptsize{(9.4)} & 161.3\scriptsize{(1.6)} & 95.9\scriptsize{(2.0)}        & 278.3\scriptsize{(16.0)} & 158.9\scriptsize{(0.6)} & 91.9\scriptsize{(7.5)}       \\
                         & 2                    & 203.8\scriptsize{(8.2)} & 161.0\scriptsize{(2.7)} & 92.7\scriptsize{(1.3)}        & 200.9\scriptsize{(13.2)} & 156.9\scriptsize{(4.7)} & 90.7\scriptsize{(4.8)}       \\ 
\hline
\multirow{2}{*}{MEP}     & 1                    & 291.7\scriptsize{(4.6)} & 161.8\scriptsize{(0.7)} & 108.8\scriptsize{(4.2)}       & 298.2\scriptsize{(5.4)}  & 157.3\scriptsize{(2.7)} & 104.6\scriptsize{(5.0)}      \\
                         & 2                    & 203.4\scriptsize{(2.0)} & 164.2\scriptsize{(2.1)} & 111.1\scriptsize{(0.7)}       & 207.8\scriptsize{(7.3)}  & 158.9\scriptsize{(3.0)} & 105.0\scriptsize{(2.2)}        \\ 
\hline
\multirow{2}{*}{TrajDiv} & 1                    & 289.3\scriptsize{(8.8)} & 150.8\scriptsize{(3.1)} & 60.1\scriptsize{(5.0)}     & 270.8\scriptsize{(2.5)}  & 142.5\scriptsize{(2.8)} & 70.1\scriptsize{(6.7)}       \\
                         & 2                    & 194.2\scriptsize{(0.7)} & 142.1\scriptsize{(2.3)} & 53.7\scriptsize{(12.4)} & 192.8\scriptsize{(8.7)}  & 137.3\scriptsize{(4.9)} & 63.8\scriptsize{(8.2)}       \\
\bottomrule
\end{tabular}
\caption{Average episode reward and standard deviation (over 5 seeds) on 3 layouts for different methods played with human proxy policies. The Pos. column indicates the roles played by AI policies.}
\label{tab:Average-reward-poolsize}
\vspace{-5mm}
\end{table}

\subsubsection{Policy Pool Construction}
\label{sec:app-poolconstruction}
{\name} has two techniques for the policy pool, i.e., (1) policy filtering to remove duplicated biased policies and (2) the use of MEP policies under the game reward for half of the pool size. 
We measure the performance with human proxies by turning these options off. For ``\emph{HSP w.o. Filtering}'', we keep all the policies by random search in the policy pool, which results in a larger pool size of 54 (18 MEP policies and a total of 36 random search ones). For``\emph{HSP w.o. MEP}'', we exclude MEP policies in the policy pool and keep all the biased policies without filtering, which leads to the same pool size of 36. The results are shown in 
Table.~\ref{tab:Average-reward-policy}. 

\begin{table}
\centering

\begin{tabular}{ccccc} 
\toprule
                                        & Pos. & Asy. Adv.   & Coor. R. & Cou. Circ.  \\ 
\midrule

\multirow{2}{*}{HSP w.o. MEP (pool size = 36)}  & 1    & 308.5\scriptsize{(4.4)}  & 157.5\scriptsize{(3.0)} & 94.0\scriptsize{(2.7)}      \\
                                        & 2    & 219.6\scriptsize{(15.9)} & 157.7\scriptsize{(2.5)} & 100.4\scriptsize{(1.1)}   \\ 
\hline
\multirow{2}{*}{HSP w.o. Filtering (pool size = 54)}       & 1    & 311.3\scriptsize{(8.1)}  & 139.2\scriptsize{(5.6)} & 80.1\scriptsize{(4.6)}    \\
                                        & 2    & 209.3\scriptsize{(4.0)}  & 138.5\scriptsize{(3.1)} & 88.7\scriptsize{(0.9)}    \\
\hline
\multirow{2}{*}{HSP (pool size = 36)}   & 1    & 300.3\scriptsize{(2.2)}  & 160.0\scriptsize{(2.6)} & 107.4\scriptsize{(3.5)}   \\
                                        & 2    & 217.1\scriptsize{(3.3)}  & 160.6\scriptsize{(3.3)} & 106.6\scriptsize{(3.0)}   \\ 
\bottomrule
\end{tabular}

\caption{Average episode reward and standard deviation (over 5 seeds) on 3 layouts for different methods played with human proxy policies. The Pos. column indicates the roles played by AI policies.}
\label{tab:Average-reward-policy}
\end{table}

\subsection{Cooperation with Scripted Policies with Strong Behavior Preferences}
Table~\ref{tab:app-Average-reward-script} illustrates average episode reward and standard deviation (over 5 seeds) in all layouts with scripted policies. All values within a difference of 5 from the maximum value are marked in bold.

\begin{sidewaystable}[h]
\centering
\begin{tabular}{cccccccccc} 
\toprule
                             & \multirow{2}{*}{Scripts} & \multicolumn{2}{c}{FCP}                   & \multicolumn{2}{c}{MEP}                     & \multicolumn{2}{c}{TrajDiv} & \multicolumn{2}{c}{HSP}                      \\ 
\cline{3-10}
                             &                          & 1                   & 2                   & 1                    & 2                    & 1           & 2             & 1                    & 2                     \\ 
\midrule
\multirow{2}{*}{Asy. Adv.}   & Oni. Pla.                & 304.0\scriptsize{(20.8)}         & 365.7\scriptsize{(5.3)}          & 295.0\scriptsize{(25.7)}          & 365.9\scriptsize{(2.8)}           & 296.5\scriptsize{(25.2)} & 350.7\scriptsize{(8.9)}    & \textbf{357.3}\scriptsize{(2.4)}  & \textbf{396.3}\scriptsize{(17.4)}  \\
                             & Oni. Pla.\&Deli.               & 364.9\scriptsize{(2.9)}          & \textbf{230.5}\scriptsize{(3.9)} & \textbf{366.8}\scriptsize{(1.5)}  & \textbf{230.3}\scriptsize{(5.3)}  & 358.2\scriptsize{(6.0)}  & 221.8\scriptsize{(3.4)}    & \textbf{371.4}\scriptsize{(3.3)}  & \textbf{229.9}\scriptsize{(4.9)}   \\ 
\hline
\multirow{2}{*}{Coor. Ring}  & Oni. Every.              & 106.0\scriptsize{(8.7)}          & 112.1\scriptsize{(7.2)}          & \textbf{124.6}\scriptsize{(3.6)}  & \textbf{123.4}\scriptsize{(3.2)}  & 115.7\scriptsize{(11.1)} & 118.1\scriptsize{(6.7)}    & \textbf{120.8}\scriptsize{(14.3)} & \textbf{121.7}\scriptsize{(10.9)}  \\
                             & Dish Every.              & 91.9\scriptsize{(4.1)}           & 97.0\scriptsize{(3.4)}           & 101.4\scriptsize{(3.1)}           & 99.0\scriptsize{(7.3)}            & 107.5\scriptsize{(5.2)}  & 107.1\scriptsize{(5.3)}    & \textbf{114.9}\scriptsize{(7.7)}  & \textbf{116.0}\scriptsize{(7.2)}   \\ 
\hline
\multirow{2}{*}{Coun. Circ.} 
& Oni. Every.              & 64.80\scriptsize{(8.8)}          & 62.7\scriptsize{(9.5)}           & 90.6\scriptsize{(3.9)}            & 87.4\scriptsize{(6.2)}            & 86.0\scriptsize{(12.1)}  & 78.1\scriptsize{(13.4)}    & \textbf{110.5}\scriptsize{(3.7)}  & \textbf{104.4}\scriptsize{(3.4)}   \\
                             & Dish Every.              & 57.9\scriptsize{(5.3)}           & 56.1\scriptsize{(5.2)}           & 53.5\scriptsize{(1.1)}            & 52.5\scriptsize{(2.4)}            & 59.3\scriptsize{(2.9)}   & 55.0\scriptsize{(1.5)}     & \textbf{79.3}\scriptsize{(3.5)}   & \textbf{77.8}\scriptsize{(4.7)}    \\
                             
                             & Oni.2Mid. Cou.              & 73.0\scriptsize{(11.5)}           & 72.4\scriptsize{(12.6)}           & 112.1\scriptsize{(6.5)}            & 110.7\scriptsize{(5.8)}            & 109.5\scriptsize{(5.6)}   & 107.9\scriptsize{(5.5)}     & \textbf{149.3}\scriptsize{(3.1)}   & \textbf{147.1}\scriptsize{(3.2)}    \\
\hline
\multirow{2}{*}{Dis. Toma.}  & Toma. Pla.               & 13.5\scriptsize{(2.6)}           & 17.8\scriptsize{(7.6)}           & 25.9\scriptsize{(17.6)}           & 14.3\scriptsize{(3.6)}            & 29.6\scriptsize{(14.1)}  & 17.2\scriptsize{(4.8)}     & \textbf{277.7}\scriptsize{(9.7)}  & \textbf{278.9}\scriptsize{(18.8)}  \\
                             & Toma. Pla.\&Deli.              & 197.5\scriptsize{(4.9)}          & 158.3\scriptsize{(7.3)}          & 185.9\scriptsize{(9.9)}           & 174.8\scriptsize{(7.4)}           & 175.3\scriptsize{(21.6)} & 154.3\scriptsize{(17.6)}   & \textbf{237.2}\scriptsize{(8.2)}  & \textbf{231.9}\scriptsize{(22.0)}  \\ 
\hline
\multirow{2}{*}{Many Ord.}   & Toma. Pla.               & 285.4\scriptsize{(19.5)}         & 279.8\scriptsize{(12.8)}         & 224.1\scriptsize{(61.2)}          & 227.6\scriptsize{(60.5)}          & 263.8\scriptsize{(10.1)} & 254.6\scriptsize{(5.7)}    & \textbf{315.7}\scriptsize{(12.1)} & \textbf{319.9}\scriptsize{(6.5)}   \\
                             & Toma. Pla.\&Deli.              & \textbf{334.1}\scriptsize{(5.6)} & \textbf{324.2}\scriptsize{(4.9)} & \textbf{329.1}\scriptsize{(12.7)} & \textbf{327.2}\scriptsize{(12.4)} & 298.8\scriptsize{(3.2)}  & 292.6\scriptsize{(1.5)}    & 326.7\scriptsize{(4.0)}           & \textbf{322.3}\scriptsize{(3.9)}   \\
\bottomrule
\end{tabular}
\caption{Average episode reward and standard deviation (over 5 seeds) in all layouts with scripted policies. All values within Within a difference of 5 from maximum reward are marked in bold. The "1" and "2" indicates the roles played by AI policies.}
\label{tab:app-Average-reward-script}
\end{sidewaystable}

\subsection{Human-AI Experiment}
\label{sec:app-human}
\subsubsection{Experiment Setting}
We recruited 60 volunteers by posting the experiment advertisement on a public platform. They are provided with a detailed introduction to the basic gameplay and the experiment process. The Overcooked game was deployed remotely on a server that the volunteers could access with their browsers. According to the feedback, over 90 percent of volunteers had no prior experience with Overcooked.  We uniformly divided 60 volunteers into 5 groups assigned to each of the 5 layouts. We designed the experiment to last around 30 minutes for each volunteer to ensure the validity of the data. Due to availabilities of volunteers, experiments are conducted within two consecutive days.
The experiment has two stages. In the first stage, which is called the \emph{warm-up} stage, the participants are encouraged to explore the behaviors of 4 given AI agents without a time limit. After the first stage, they are required to comment on their game experience, e.g., whether the AI agents are cooperative and comfortable to play with, and rank the agents accordingly. In the second stage, each participant is instructed to achieve scores as high as they could in 24 games (4 AI agents $\times$ 2 player positions $\times$ 3 repeats). 

We remark that, on the environment side, different from human-AI experiments performed by prior works~\citep{zhao2021maximum, carroll2019utility} in Overcooked, we slow down the AI agents so that the AI agents have similar speed with human players. More specifically, 7 idle steps are inserted before each step of the AI agent. Such an operation is necessary since, in our prior user studies, we find that human players commonly feel uncomfortable if the AI agent is much faster and human players could contribute little to the score.

\subsubsection{Human Feedback}
\label{sec:app-humanfeedback}

We collected and analyzed the feedback from the participants to see how they felt playing with AI agents. Here we summarize the typical reflections.
\begin{enumerate}
    \item In \emph{Coordination Ring}, the most annoying thing reported is players blocking each other during movement. To effectively maneuver in the ring-like layout, players must reach a temporary agreement on either going clockwise or counterclockwise. {\name} is the only AI able to make way for the other player, while others can not recover by themselves once stuck. For example, both FCP and TrajDiv players tend to take a plate and wait next to the pot immediately after one pot is filled. But they can neither take a detour when blocked on their way to the dish dispenser nor yield their position to the human player trying to pass through. The video recorded in the human study can be found in Part 4.2 of \url{https://sites.google.com/view/hsp-iclr}.
    \item In \emph{Counter Circuit}, one efficient strategy is passing onion via the counter in the middle of the room: a player at the bottom fetches onions and places them on the counter, while another player at the top picks up the onions and puts them into pots. {\name} is the only AI player capable of this strategy in both top and bottom places and performs the highest onion passing frequency cooperating with human players as shown in Figure.~\ref{fig:real-human-random3}.
    \item In \emph{Distant Tomato}, one critical thing is that mixed (onion-tomato) soups give no reward, which means two players need to agree on the soup to cook. All AI agents perform well when the other player focuses on onion soups. However, all AI agents except for {\name} fail to deal with tomato-preferring partners as shown in Table.~\ref{tab:distant-tomato}. FCP, MEP, or TrajDiv agents never actively choose to place tomatoes and keep placing onions even when a pot has tomatoes in it, resulting in invalid orders. On the contrary, {\name} chooses to place tomatoes when there are tomatoes in the pot. Participants commonly agree that the {\name} agent is the best partner to play with in this layout. The video recorded in the human study can be found in Part 4.2 of \url{https://sites.google.com/view/hsp-iclr}.
    \item In \emph{Many Orders}, most participants claim that {\name} is able to pick up soups from all three pots, while other AI agents only concentrate on the pot in front of them and ignore the middle pot even if the human player attempts to use it. Table. ~\ref{tab:many-orders} shows that {\name} agent picks up most soups from the middle pot and meanwhile gets the highest average episode reward.
\end{enumerate}

\subsubsection{Human Preference on Different AI Agents}
\label{sec:app-humanrank}
Figure.~\ref{fig:human-ranking} illustrates human preference for different AI agents. In all layouts except a relatively restricted and simple layout, \emph{Coordination Ring}, human players strongly prefer {\name} over other AI agents. In \emph{Coordination Ring}, though human players rank MEP above {\name}, {\name} is still significantly better than FCP and TrajDiv.

\textbf{Calculation Method}: Human preference for different methods is computed as follows. Assume we are comparing human preference between method A and method B. Let $N$ be the total number of human players attending the experiments in one layout, $N_A$ be the number of human players who rank A over B, and $N_B$ be the number of those who rank B over A. "Human preference for method A over method B" is computed as $\frac{N_A}{N} - \frac{N_B}{N}$.

\begin{figure}[h]
\vspace{-3mm}
	\centering
 	\subfloat[Asy. Adv.]{
 		\includegraphics[height=2.68cm]{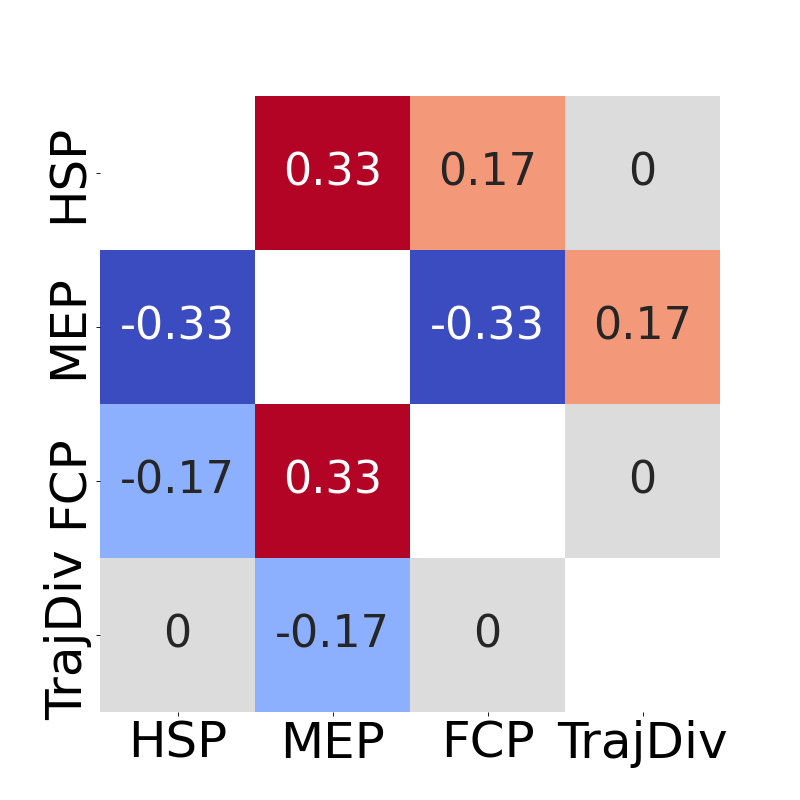}
 		}
 	\subfloat[Coor. Ring]{
 		\includegraphics[height=2.68cm]{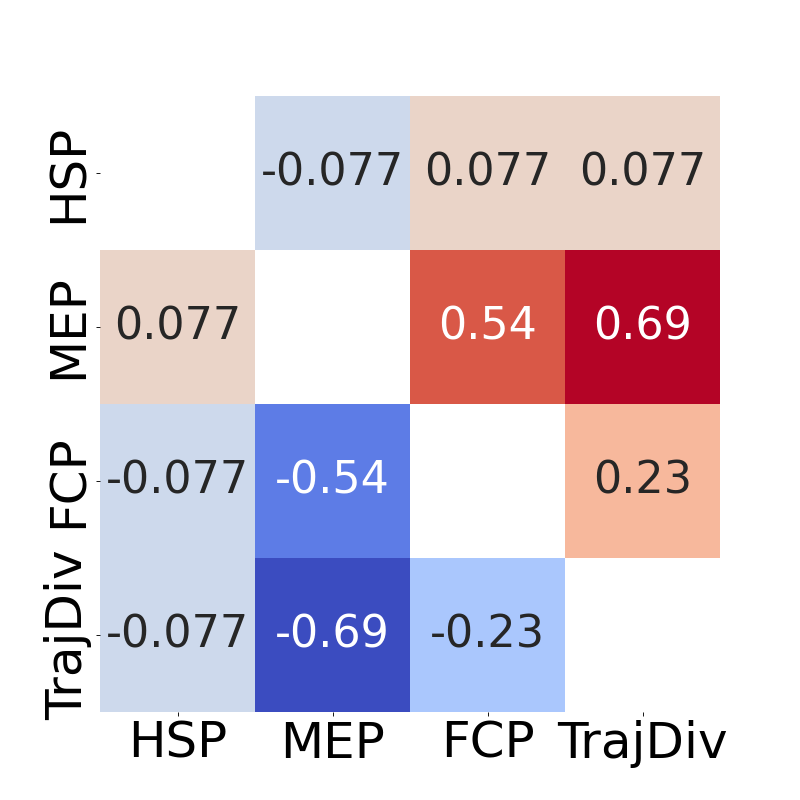}
 		}
 	\subfloat[Coun. Circ.]{
 		\includegraphics[height=2.68cm]{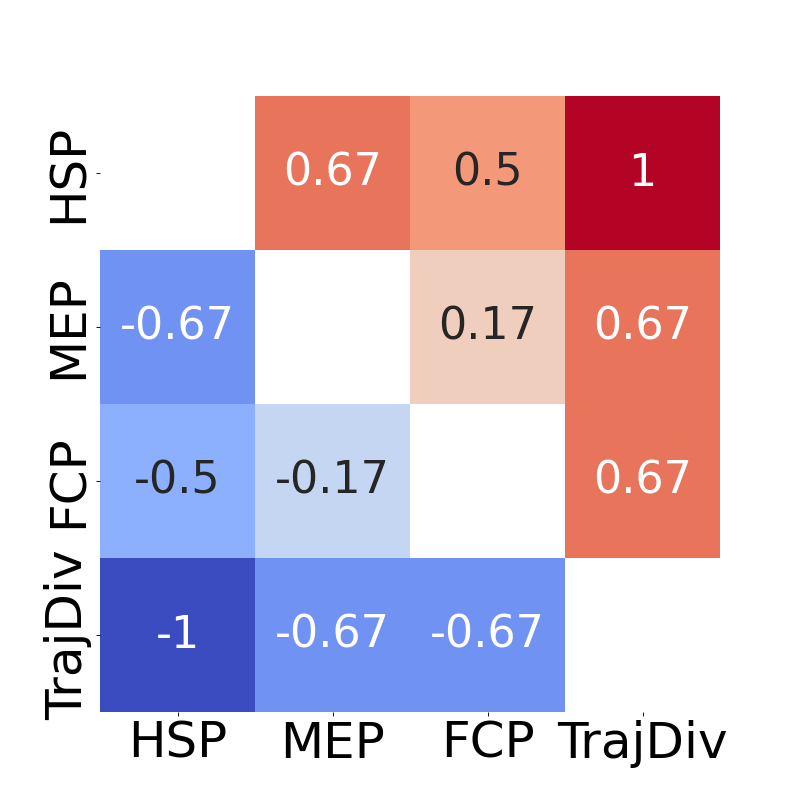}
 		}
 	\subfloat[Dis. Toma.]{
 		\includegraphics[height=2.68cm]{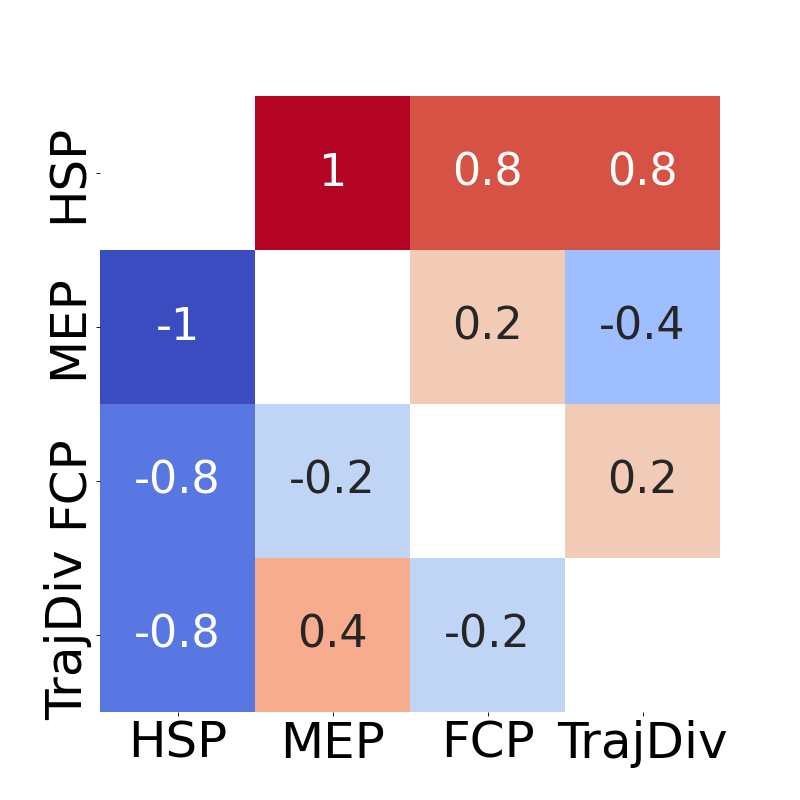}
 		}
 	\subfloat[Many Ord.]{
 		\includegraphics[height=2.68cm]{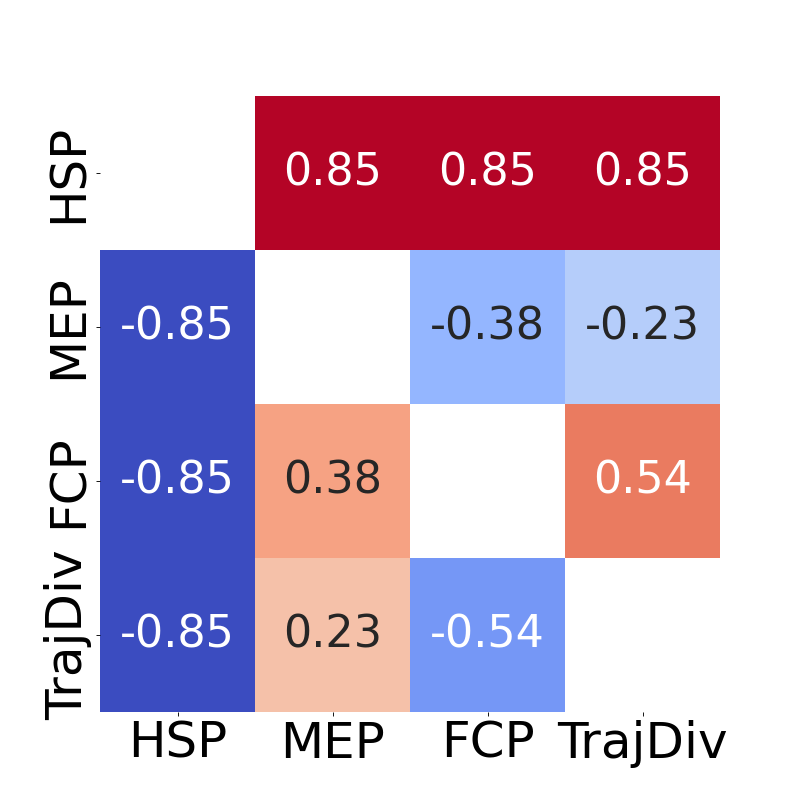}
 		}
	\centering 
	\vspace{-2mm}
	\caption{Human preference for row partner over column partner in all layouts.}
    \label{fig:human-ranking}
\end{figure}

\subsubsection{Scores in the Second Stage}

Table \ref{tab:Average-reward-human} shows average \emph{reward per episode} during the second stage in all layouts. All methods have comparable episode rewards in Asymm. Adv and Coord. Ring. There is no room for improvement since all the methods have reached the highest possible rewards. In Counter Circ., the most complex layout in this category, HSP achieves a better performance than baselines: HSP has a 155+ reward while the most competitive baseline MEP has a reward of 134+. We remark that the reward difference between HSP and MEP is around 20, which is exactly the value of 1 onion soup delivery. This implies that the {\name} agent can, on average, deliver one more soup than all the baselines per game episode with humans, which is a significant improvement.

\begin{table}[h]
\centering
\begin{tabular}{ccccccc} 
\toprule
                                 & Pos. & Asy. Adv.           & Coor. Ring           & Cou. Circ.   & Dis. Toma.        & Many Ord.           \\ 
\midrule
\multirow{2}{*}{FCP}     & 1       & \textbf{339.3}\scriptsize{(38.17)}  & 185.0\scriptsize{(19.73)}     & 127.7\scriptsize{(28.14)}                                       & \textbf{351.3}\scriptsize{(82.25)}  & 312.4\scriptsize{(58.73)}\\ 
                                    & 2       & 321.3\scriptsize{(34.80)}          & 180.7\scriptsize{(22.98)}      &118.3\scriptsize{(29.20)}    & \textbf{320.5}\scriptsize{(66.49)}   &321.3\scriptsize{(61.12)} \\ 
\hline
\multirow{2}{*}{MEP}     & 1       & 329.7\scriptsize{(45.97)}         & \textbf{193.3}\scriptsize{(22.11)} & 136.9\scriptsize{(27.00)}                                     & 341.9 \scriptsize{(65.07)}  & 322.0\scriptsize{(50.53)}\\ 
                                 & 2       & \textbf{324.2}\scriptsize{(39.93)}          & 183.6\scriptsize{(26.75)}  &134.5\scriptsize{(28.63)}     & 313.9\scriptsize{(78.29)} & 319.2\scriptsize{(52.98)}\\ 
\hline
\multirow{2}{*}{TrajDiv} & 1       & 329.0\scriptsize{(43.18)}         & 184.7\scriptsize{(28.60)}          & 112.6\scriptsize{(25.78)}                                               & 327.1\scriptsize{(71.04)} & 312.9\scriptsize{(62.82)}\\ 
                                 & 2       & 318.8\scriptsize{(48.97)}          & 176.8\scriptsize{(31.33)}          & 105.0\scriptsize{(31.05)}            & 316.0\scriptsize{(77.65)} & 334.2\scriptsize{(57.99)}\\ \hline
\multirow{2}{*}{HSP}     & 1       & 336.0\scriptsize{(35.55)}    & 185.5\scriptsize{(38.92)}    &\textbf{158.0}\scriptsize{(28.56)}                                       & 331.6\scriptsize{(61.33)} & \textbf{384.3}\scriptsize{(47.50)}\\ 
                                 & 2       & 318.8\scriptsize{(48.97)}  & \textbf{188.9}\scriptsize{(22.00)}  & \textbf{155.2}\scriptsize{(23.43)}   & 305.9\scriptsize{(58.61)}  & \textbf{380.7}\scriptsize{(62.27)}\\
\bottomrule
\end{tabular}
\caption{Average reward per episode in all layouts with human players in the second stage.}
\label{tab:Average-reward-human}
\end{table}
\end{document}